\documentclass[letterpaper]{article} % DO NOT CHANGE THIS
\usepackage[preprint]{aaai2027}
\usepackage[hyphens]{url}  % DO NOT CHANGE THIS
\usepackage{graphicx} % DO NOT CHANGE THIS
\usepackage{natbib}  % DO NOT CHANGE THIS AND DO NOT ADD ANY OPTIONS TO IT
\usepackage{caption} % DO NOT CHANGE THIS AND DO NOT ADD ANY OPTIONS TO IT
\usepackage{algorithm}
\usepackage{algorithmic}
\usepackage[switch]{lineno}
\usepackage{amsmath}   % align, equation, \sim, \tfrac, etc.
\usepackage{amssymb}   % \mathbb, \Pr-style symbols
\usepackage{amsthm}    % theorem/definition/proof environments

\newcommand{\xmark}{\ensuremath{\boldsymbol{\times}}}

\usepackage{booktabs}   % \toprule, \midrule, \bottomrule
\usepackage{multirow}   % \multirow

\theoremstyle{plain}
\newtheorem{theorem}{Theorem}
\newtheorem{lemma}{Lemma}
\newtheorem{proposition}{Proposition}

\theoremstyle{definition}

\newtheorem{definition}{Definition}
\theoremstyle{remark}

\title{%
\raisebox{-0.35\height}{\includegraphics[height=1.8em]{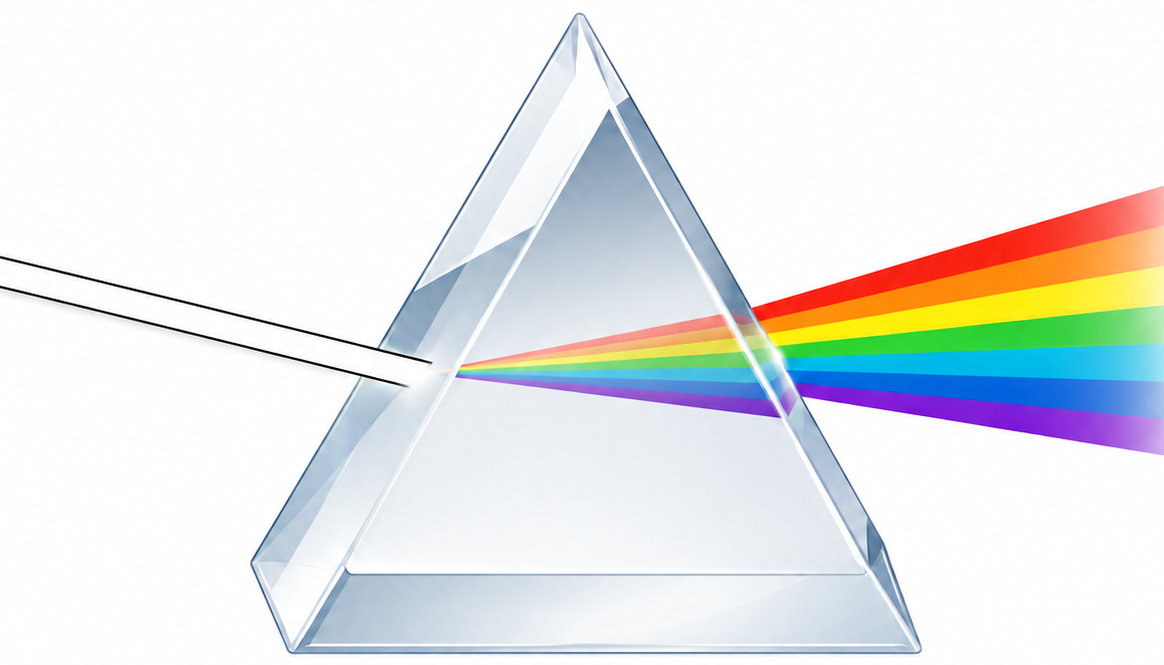}}%
Prism-GRPO: Faster VLA Policy Optimization \\
via Splitting Same-outcome Groups%
}

\author{
    Zeyun Deng$^{1}$\thanks{Work was done during an internship at AWS AI.},
    Yuzhe Lu$^{2}$,
    Yawei Wang$^{2}$, 
    Linbo Liu$^{2}$, \\
    Qing Ping$^{2}$,
    Han Ding$^{2}$,
    Guande Wu$^{2}$,
    Panpan Xu$^{2}$,
    Jun Huan$^{2}$
}
\affiliations{
    $^{1}$Purdue University 
    $^{2}$AWS AI
}

\begin{document}

\maketitle

\begin{abstract}
% zeyun: this is your original version since i don't quite like my first editing attempt below. will think more

% Reinforcement learning improves vision-language-action (VLA) policies beyond supervised
% fine-tuning (SFT), and GRPO has become a popular choice for this stage: it scores rollouts
% relative to others in the same group. With a binary success reward, however, groups whose
% rollouts all succeed or all fail have no advantage spread, produce no gradient, and are
% discarded. Such groups are common, dominating early training when the policy fails almost
% every scene, so much of an expensive
% rollout budget is thrown away. 
% We propose Prism-GRPO, which augments the
% binary outcome with a weighted continuous execution-quality term. This splits same-outcome
% groups into finer reward levels while preserving success dominance: every success still
% outranks every failure.
% Unlike task-specific progress rewards, Prism-GRPO only requires a trajectory-level quality score,
% making the framework reusable across tasks and quality notions. 
% We obtain such scores from
% simulator data, executed action trajectories, or vision-language models.
%  When aligned with successful execution, quality improves task success as well as behavior. 
% We provide theoretical guarantee that Prism-GRPO wastes fewer rollouts and characterize when it preserves progress on
% success.
% On RoboTwin, it reaches target success rates with up to 56\% fewer rollouts and improves
% execution quality; in real-world tests, it avoids a simulator shortcut that often fails on
% hardware while maintaining clean success.

GRPO is increasingly used for reinforcement learning of vision-language-action (VLA) policies because, unlike PPO, it does not require training a critic. This simplification comes with a sampling cost: group-relative advantages require multiple rollouts from each scene. Under binary success rewards, groups whose rollouts all succeed or all fail have zero advantage and are discarded by dynamic sampling. These groups are especially common early in training, when most rollouts fail, wasting much of the expensive robotic rollout budget. We introduce Prism-GRPO, which augments binary outcome reward with a weighted trajectory-level execution-quality score. By splitting same-outcome groups into a quality spectrum, Prism-GRPO recovers training signal while ensuring that every success still outranks every failure. Quality scores can be derived from simulator contacts, executed actions, or visual observations, avoiding task-specific progress rewards. We prove that Prism-GRPO never increases the probability that a sampled group is discarded for having zero advantages, and derive a gradient-alignment condition under which its combined update remains a local ascent direction for task success. Across four RoboTwin tasks spanning different horizons and coordination patterns, Prism-GRPO improves success and quality at matched rollout budgets and reaches target success rates with up to 56\% fewer rollouts. It also suppresses a reward-hacking shortcut, with the cleaner behavior transferring under direct deployment to a real robot. Through ablations, we show consistent gains across contact-, smoothness-, and VLM-derived quality signals.
\end{abstract}
\section{Introduction}

Reinforcement learning is a standard way to improve vision-language-action (VLA) policies beyond
supervised fine-tuning, and GRPO~\citep{shao2024deepseekmath} is attractive because it avoids a
learned critic by comparing rollouts within a group. Under a binary success reward, however, GRPO
can be highly sample-inefficient. If all trajectories in a group succeed, or all of them fail,
every rollout receives the same reward. The group has no advantage spread, produces no gradient,
and is discarded by dynamic sampling \cite{yu2026dapo}. GRPO therefore learns only from mixed-outcome groups,
wasting the samples collected in degenerate groups.

\begin{figure}[!t]
\centering
\includegraphics[width=\columnwidth]{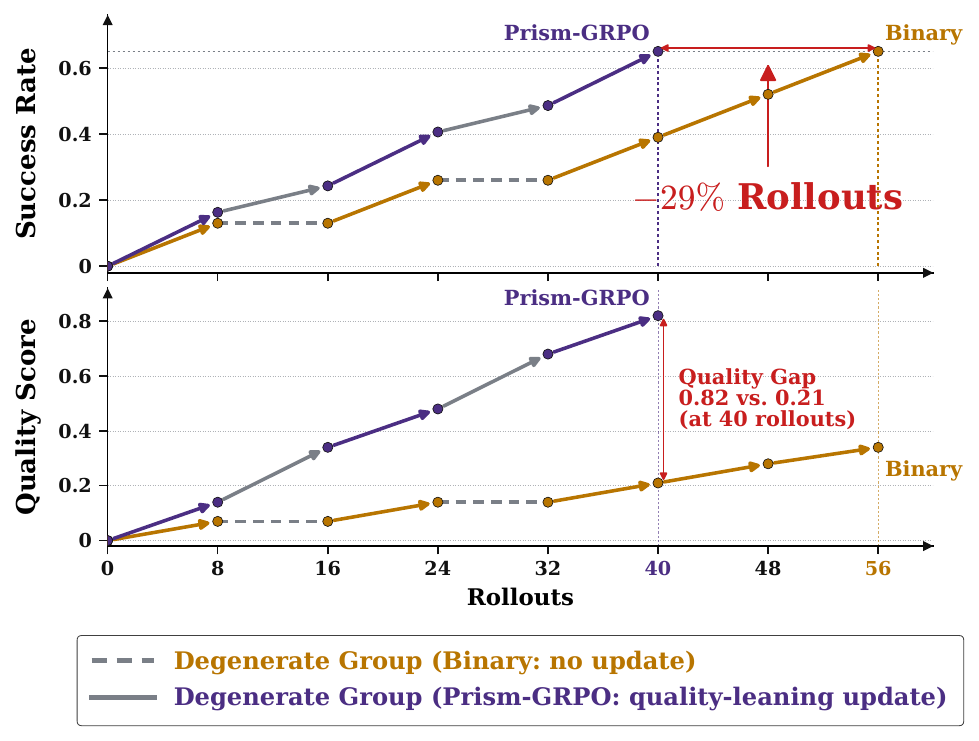}
\caption{
\textbf{Illustrative overview.}
Binary GRPO discards all-success and all-failure groups because identical rewards yield zero
advantage. Prism-GRPO splits them by execution quality to recover informative, success-aligned
updates.
}
\label{fig}
\end{figure}

Recent LLM reasoning work revives degenerate groups using finer-grained signals. RL-ZVP uses the
policy's output entropy~\citep{le2025no}, CAST uses self-teacher guidance~\citep{li2026cast}, and
LLM-as-a-Verifier~\citep{kwok2026llm} uses an external verifier's scoring-token logits. In
robotics, the execution itself provides a natural signal. Two rollouts with the same outcome may
differ in unintended contact, object disturbance, or motion smoothness. Unlike task progress,
these properties describe how the task is executed rather than how far it has advanced, making
them reusable across manipulation tasks. The only requirement is an observable trajectory score
aligned with successful execution, leaving its definition and measurement flexible.
We turn this observation into \textbf{Prism-GRPO}, \textbf{P}olicy optimization by
\textbf{R}efining \textbf{I}dentical-outcome groups into a quality \textbf{S}pectrum for
\textbf{M}ore efficient GRPO. 
As illustrated in Figure~\ref{fig}, Prism-GRPO uses execution quality to turn otherwise discarded
same-outcome groups into informative updates. Breaking reward ties alone, however, does not guarantee progress toward task success. Prism-GRPO therefore preserves success dominance by construction; we further characterize when the recovered signal advances success.
Our contributions are threefold:
% \begin{itemize}
% \item We propose Prism-GRPO, which abstracts trajectory quality across definitions and sources to
% recover degenerate groups while preserving success dominance.
% \item We show theoretically that Prism-GRPO improves sample efficiency and characterize when
% quality preserves progress on success.
% \item We show that Prism-GRPO improves sample efficiency and  quality in simulation, with real-world validation.
% \end{itemize}

\begin{itemize}
\item We introduce Prism-GRPO for reinforcement learning of VLA policies on robotic tasks, combining binary task success with bounded trajectory quality to recover training signal from same-outcome rollouts.

\item We prove that Prism-GRPO never increases the expected number of rollouts
needed to obtain an informative group. Under success--quality gradient alignment, its combined update is a first-order ascent direction for both objectives.

\item Across four RoboTwin tasks, Prism-GRPO reaches target
success rates with up to 56\% fewer rollouts while improving execution quality.
The gains hold across contact-, smoothness-, and VLM-derived signals, and cleaner behavior transfers under direct real-robot deployment.
\end{itemize}

% \noindent\textbf{Contributions.}
% \textbf{(1)} We introduce Prism-GRPO for reinforcement learning of VLA policies
% on robotic tasks, using bounded trajectory quality to learn from same-outcome
% groups.
% \textbf{(2)} We prove that it never increases the expected rollouts needed to
% obtain an informative group and characterize when its combined update ascends
% both success and quality.
% \textbf{(3)} Across four RoboTwin tasks, Prism-GRPO reaches target success rates
% with up to 56\% fewer rollouts, improves quality across contact-, smoothness-,
% and VLM-derived signals, and transfers cleaner behavior to a real robot.

\section{Related Work}
\label{sec:related}
\paragraph{Vision-language-action models.}
Vision-language-action (VLA) models map visual observations and a language instruction to a
robot action, typically the target pose and gripper state across the arm's degrees of freedom,
on top of a pretrained vision-language backbone. They split by how that action is represented.
OpenVLA~\citep{kim2024openvla,kim2025fine} discretizes each degree of freedom into bins and emits
one token per degree of freedom, so control becomes next-token prediction over a fixed action
vocabulary; the $\pi$ models~\citep{black2024pi_0,intelligence2025pi} instead use a flow-matching
head that outputs continuous actions directly. Pretrained by imitation, these policies are capped
at their demonstration quality; reinforcement learning is the standard fine-tuning stage that
lifts them past it.

\paragraph{Reinforcement learning for VLA policies.}
RL for continuous control has long relied on PPO~\citep{schulman2017proximal}, which
estimates advantages with a separate value network. GRPO~\citep{shao2024deepseekmath}
removes this critic, scoring a group of trajectories on the same input against their mean
reward, and has driven much of the progress in LLM post-training. Carrying it over to VLA
policies is harder than it looks: \citet{liu2026can} find that a single trajectory-level
reward cannot fairly credit individual actions when each physically moves the agent.
SimpleVLA-RL~\citep{li2025simplevla} nonetheless makes GRPO effective for VLAs by adapting
the rollout process to encourage exploration, through dynamic sampling that discards and
resamples same-outcome groups, a raised upper clipping bound (clip-higher,
following DAPO~\citep{yu2026dapo}), a higher rollout temperature, and no KL penalty. We build
directly on this recipe.

\paragraph{Enriching the binary reward.}
The binary reward used by SimpleVLA-RL~\citep{li2025simplevla} is cheap to check but blind
within an outcome: all successes score alike and so do all failures. A common way to add signal
is to replace or augment success with graded feedback. RoboReward~\citep{lee2026roboreward}
learns a vision-language reward model from robot data, while stage-level process rewards such as
StARe-VLA~\citep{xu2025stare} use task-specific intermediate scores.
LLM-as-a-Verifier~\citep{kwok2026llm} provides another source of graded feedback, computing
continuous verifier scores from scoring-token logits rather than discrete judge outputs and
using these scores as dense rewards.

\paragraph{Reviving degenerate groups.}
Under binary rewards, same-outcome GRPO groups have zero advantage spread and provide no update.
Recent LLM reasoning work asks whether these zero-variance groups can still provide useful signal,
but differs in where that signal comes from. \textbf{RL-ZVP}~\citep{le2025no} is the closest
baseline to ours and the first to treat zero-variance groups themselves as useful training data,
reviving them with entropy from the policy's own output distribution.
Other methods
add extra assistance. \textbf{EDGE-GRPO}~\citep{zhang2025edge} combines entropy-driven advantages
with guided error correction that supplies failed samples with oracle reference solutions,
introducing additional supervision beyond the sampled rollouts. \textbf{CAST}~\citep{li2026cast}
uses an additional self-teacher pass for token-level guidance, while
\textbf{LLM-as-a-Verifier}~\citep{kwok2026llm} uses an external verifier's scoring-token logits
to produce continuous scores rather than discrete judge labels. These methods are designed around
LLM reasoning, where intermediate reasoning quality is hard to measure directly. VLA rollouts are
physical executions, so their trajectories already expose quality signals such as collisions,
object disturbance, or smoothness. 
Prism-GRPO directly ranks same-outcome rollouts using trajectory signals, without oracle references,
helper models, or policy-confidence proxies.

\paragraph{Gradient conflict in multi-objective learning.}
Two objectives on one policy may have gradients that oppose each other, so the averaged update
improves one at the other's expense. Prior work addresses this by gating or projecting auxiliary
gradients \citep{du2018adapting,yu2020gradient}, or by controlling their influence through learned
weights or strict priorities~\citep{gupta2023behavior,berducci2021hierarchical}. These methods
assume conflict. We instead target manipulation settings where higher-quality behavior tends to
support task completion, allowing the two rewards to be added directly.

\section{Preliminary}
\paragraph{GRPO with binary reward and compositional reward.}
\label{sec:prelim-grpo}
We cast VLA control as a Markov decision process $(\mathcal{S},\mathcal{A},P,R,\rho_0)$: at
each step the policy $\pi_\theta$ observes a state $s_t\in\mathcal{S}$ (images and the
instruction), takes an action $a_t\in\mathcal{A}$ (a robot command), and the environment
transitions by $P(s_{t+1}\mid s_t,a_t)$. A full episode is a trajectory
$\tau=(s_0,a_0,\dots)$, and RL seeks the $\theta$ maximizing the expected return
$\mathbb{E}_{\tau\sim\pi_\theta}[R(\tau)]$. 
GRPO optimizes this objective without a learned value function: on each scene it samples a
\emph{group} of $G$ trajectories from $\pi_\theta$, scores them $R_1,\dots,R_G$, and forms each
advantage by subtracting the group's baseline reward, so the signal is purely
relative~\citep{shao2024deepseekmath}. The reward itself is left to the practitioner.

\begin{definition}[Degenerate group]
\label{def:degen}
A group is degenerate if $R_1=\cdots=R_G$, yielding zero advantages and no gradient; otherwise,
it is informative.
\end{definition}

\begin{definition}[Objective gradients and alignment]
\label{def:shared}
Let $\pi_\theta$ have parameters $\theta\in\mathbb{R}^{d}$, and define the success and quality
objectives as
$J(\theta)=\mathbb{E}_{x,\tau}[\mathrm{success}(\tau)]$ and
$Q(\theta)=\mathbb{E}_{x,\tau}[q(\tau)]$. Their gradients
$g_{\mathrm{success}}=\nabla_\theta J(\theta)$ and
$g_{\mathrm{quality}}=\nabla_\theta Q(\theta)$ lie in the same parameter space. 
They \emph{align} if
$\langle g_{\mathrm{success}},g_{\mathrm{quality}}\rangle\ge0$ and \emph{conflict} otherwise.
\end{definition}

\section{Prism-GRPO}
\label{sec:theory}

\subsection{Reward design and problem formulation}
\label{sec:formulation}

\paragraph{Tasks, scenes, trajectories.}
A task $T$ (e.g.\ stacking two cubes) induces a distribution over scenes; we write
$x\sim\mathrm{Scene}(T)$ for a scene, consisting of an initial configuration together with its
instruction. Running $\pi_\theta$ on $x$ produces a trajectory
$\tau\sim\pi_\theta(\cdot\mid x)$, scored by a binary outcome
$\mathrm{success}(\tau)\in\{0,1\}$ and a continuous quality score $q(\tau)\in[0,1]$, where
$q(\cdot)$ is a measurement of quality, such as smoothness, larger the better. A
GRPO group drawns $G$ trajectories on the same scene $x$.

\begin{algorithm}[t]
\caption{Prism-GRPO}
\label{alg:qgrpo}
\begin{algorithmic}[1]
\REQUIRE $\pi_\theta$, scene $x$, group size $G$, $\lambda\in(0,1)$
\REPEAT
    \STATE sample $\tau_1,\dots,\tau_G \sim \pi_\theta(\cdot\mid x)$
    \STATE $R_i \leftarrow \mathrm{success}(\tau_i)\,\textcolor{red}{{}+\,\lambda\, q(\tau_i)}$
\UNTIL{not all $R_i$ are equal} \COMMENT{\textcolor{blue}{discard $G$ rollouts per retry}}
\STATE $\hat{A}_i \leftarrow R_i - \tfrac{1}{G-1}\textstyle\sum_{j\neq i} R_j$
\STATE update $\theta$ using $\{\hat{A}_i\}$
\end{algorithmic}
\end{algorithm}

\paragraph{Quality-augmented reward.}
Prism-GRPO starts with the binary outcome reward and adds execution quality. For a trajectory
$\tau$, we define
\begin{equation}
\label{eq:rewards}
R_{\mathrm{combined}}(\tau)
=
\mathrm{success}(\tau)+\lambda q(\tau),
\qquad 0<\lambda<1 .
\end{equation}
This turns tied binary rewards into a success-dominant reward spectrum. As illustrated in
Figure~\ref{fig:prism-teaser}, an all-failure group has all rewards equal to $0$ under the
binary reward, and vise versa; therefore, both have no
 spread of advantages. The quality term breaks these ties by refining failures within $[0,\lambda]$
and successes within $[1,1+\lambda]$. Since $\lambda<1$, the two bands remain separated: every
success still scores above every failure, while quality only reorders trajectories within an
outcome class. We pair this reward with the leave-one-out (RLOO) advantage estimator
\citep{kool2019buy,ahmadian2024back}, which subtracts the siblings' mean but does not divide by
the group standard deviation. This matters on same-outcome groups, where the entire reward
spread is $\lambda q$: standard GRPO normalization would cancel the scale of $\lambda$
\citep{arora2026training}, while RLOO preserves the raw quality differences.

\paragraph{Analysis target.}
Our goal is to improve rollout efficiency without sacrificing task success, while also improving execution quality. Relative to Binary GRPO, Prism-GRPO should reach the same target value of the success objective \(J(\theta)\) using no more generated rollouts and, at that point, attain at least as large a value of the quality objective \(Q(\theta)\).

\subsection{When Quality Supports Success}
\label{sec:alignment}
% We next establish a sufficient condition for quality to improve both success and execution quality.

\paragraph{Task-level correlation remains informative.}
The quality term is useful only if higher-quality trajectories are also more likely to succeed.
We measure this relationship directly from rollouts using the task-level correlation $\kappa_2=\mathrm{Corr}_{x\sim\mathrm{Scene}(T),\,\tau\sim\pi_\theta(\cdot\mid x)}
\!\left[q(\tau),\mathrm{success}(\tau)\right]$,
where $x$ denotes a sampled scene, $\tau$ a trajectory generated by the current policy
$\pi_\theta$, $q(\tau)$ the trajectory quality, and
$\mathrm{success}(\tau)\in\{0,1\}$ the task outcome.
Although many GRPO groups become degenerate, $\kappa_2$ remains meaningful because it is computed
over all scenes rather than within a single group. Its covariance $\mathrm{Cov}(q,\mathrm{success})$ decomposes as
\begin{equation}
\label{eq:totalcov}
\begin{aligned}
\underbrace{\mathbb{E}_{x}
\!\left[\mathrm{Cov}(q,\mathrm{success}\mid x)\right]}_{\text{within-scene}}
+
\underbrace{\mathrm{Cov}_{x}
\!\left(\bar q(x),\bar s(x)\right)}_{\text{between-scene}},
\end{aligned}
\end{equation}
where
$\bar q(x)=\mathbb{E}[q(\tau)\mid x]$
and
$\bar s(x)=\mathbb{E}[\mathrm{success}(\tau)\mid x]$
are the average quality and success rate on scene $x$.
The within-scene term vanishes when a scene is saturated (all rollouts share the same outcome),
but the between-scene term can remain positive because the policy tends to achieve higher quality
on easier scenes where it succeeds more often. Consequently, $\kappa_2$ continues to capture the
overall quality--success relationship even when many GRPO groups are degenerate.

\paragraph{From correlation to gradient alignment.}
Our theory assumes population-level alignment,
$\langle g_{\mathrm{success}},g_{\mathrm{quality}}\rangle\ge0$, where
$g_{\mathrm{success}}=\nabla_\theta\mathbb{E}[\mathrm{success}]$ and
$g_{\mathrm{quality}}=\nabla_\theta\mathbb{E}[q]$. 
To study this condition, we analyze a single softmax decision and identify a sufficient mechanism
for alignment.
Let $\rho(s)$ denote the policy-weighted correlation between the actions' success and quality values.
The proposition below shows that sufficiently high $\rho(s)$ yields decision-level alignment.

\begin{figure}[!t]
\centering
\includegraphics[width=\columnwidth]{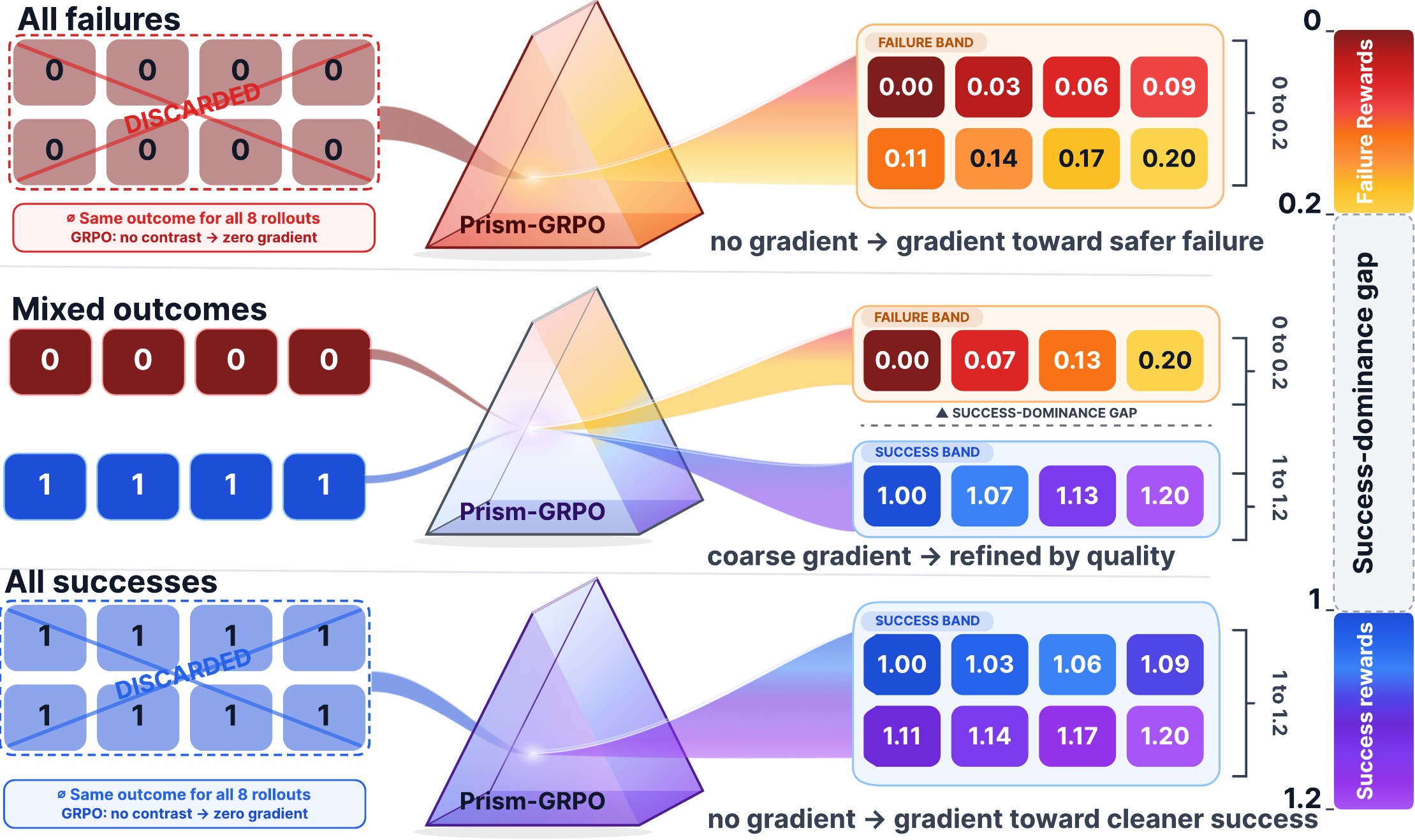}
\caption{\textbf{Prism-GRPO splits same-outcome groups into a reward spectrum.} Colors indicate rewards. $\lambda$=0.2.}
\label{fig:prism-teaser}
\end{figure}

% \begin{proposition}
% \label{prop:corr-to-align}
% Consider a single softmax decision at a state $s$, over $K$ actions with distribution
% $\pi=\pi_\theta(\cdot\mid s)$.
% Let $u_k$ and $v_k$ denote the success and quality values of action $k$, respectively. Let
% $\rho(s)=\mathrm{Corr}_{k\sim\pi}(u_k,v_k)$
% be their policy-weighted correlation.
% Assume both $u$ and $v$ take at least two distinct values under $\pi$.
% Let $\kappa_\pi\ge1$ be the conditioning number of $\pi$, which measures how unevenly probability
% is distributed across actions ($\kappa_\pi=1$ for a uniform policy and increases as the policy
% becomes more concentrated). Appendix~\ref{app:corr-to-align} gives its precise definition and the
% proof.
% Writing
% $g_{\mathrm{success}}(s)$
% and
% $g_{\mathrm{quality}}(s)$
% for the policy gradients of this decision,
% \begin{equation}
% \label{eq:corr-threshold}
% \rho(s)\ge
% \frac{\kappa_\pi-1}{\kappa_\pi+1}
% \Longrightarrow
% \langle
% g_{\mathrm{success}}(s),
% g_{\mathrm{quality}}(s)
% \rangle
% \ge0.
% \end{equation}
% \end{proposition}

\begin{proposition}
\label{prop:corr-to-align}
Consider a single softmax decision at state $s$ over $K$ actions, with distribution
$\pi=\pi_\theta(\cdot\mid s)$. Let $u_k$ and $v_k$ denote the success and quality values of
action $k$, and let $\rho(s)=\mathrm{Corr}_{k\sim\pi}(u_k,v_k)$ be their policy-weighted
correlation. Assume both $u$ and $v$ take at least two distinct values under $\pi$.
Let $D_\pi=\operatorname{diag}(\pi)$, and let $\kappa_\pi\ge1$ be its condition number on the
policy-centered subspace. It measures how unevenly probability is distributed across actions:
$\kappa_\pi=1$ for a uniform policy and can increase as the policy becomes more concentrated.
Appendix~\ref{app:corr-to-align} gives the precise definition and proof. Writing
$g_{\mathrm{success}}(s)$ and $g_{\mathrm{quality}}(s)$ for the corresponding policy gradients,
\begin{equation}
\label{eq:corr-threshold}
\rho(s)\ge
\frac{\kappa_\pi-1}{\kappa_\pi+1}
\Longrightarrow
\left\langle
g_{\mathrm{success}}(s),
g_{\mathrm{quality}}(s)
\right\rangle
\ge0.
\end{equation}
\end{proposition}

\paragraph{Interpretation and scope.}
The threshold $(\kappa_\pi-1)/(\kappa_\pi+1)$ depends on policy conditioning: it is $0$ for a
uniform policy and approaches $1$ only under extreme concentration. Modern RL methods discourage
such collapse through clip-higher, temperature sampling, and entropy regularization
\citep{yu2026dapo,cui2025entropy,simoni2025gtpo}; our implementation inherits clip-higher and
temperature-$1$ sampling from SimpleVLA-RL~\citep{li2025simplevla}. For example,
$\pi=(0.97,0.015,0.015)$ has $\kappa_\pi\approx2.9$, requiring
$\rho(s)\approx0.49$ for decision-level alignment. This gives a tractable alignment mechanism at
a single decision, but does not establish population-level alignment across scenes, states, and
stochastic long-horizon trajectories. Finite rollouts yield noisy estimates, and reliable
verification would require many batches with full-model gradient aggregation. We therefore retain
$\langle g_{\mathrm{success}},g_{\mathrm{quality}}\rangle\ge0$ as an explicit assumption, report
correlation diagnostics in Appendix~\ref{app:correlation}, and evaluate it empirically in full
Prism-GRPO training.

\subsection{The combined update points toward success}
\label{sec:direction}
Under the alignment condition in Section~\ref{sec:alignment}, the combined gradient improves quality without sacrificing success.
\begin{proposition}[The combined gradient ascends both success and quality]
\label{prop:align}
Whenever the population gradients align,
$\langle g_{\mathrm{success}},g_{\mathrm{quality}}\rangle\ge0$, the combined direction
$g_{\mathrm{combined}}=g_{\mathrm{success}}+\lambda g_{\mathrm{quality}}$ is a first-order
ascent direction for both the success objective $J(\theta)$ and the quality objective
$Q(\theta)$. In particular,
$\langle\nabla_\theta J,g_{\mathrm{combined}}\rangle
\ge\|g_{\mathrm{success}}\|^2$ and
$\langle\nabla_\theta Q,g_{\mathrm{combined}}\rangle
\ge\lambda\|g_{\mathrm{quality}}\|^2$.
Both are non-negative and are strictly positive whenever the corresponding gradient is nonzero.
The proof is given in Appendix~\ref{app:align}.
\end{proposition}

Proposition~\ref{prop:align} shows that $g_{\mathrm{combined}}$ preserves first-order progress on
success, with directional derivative at least $\|g_{\mathrm{success}}\|^2$, while also improving
quality. Thus, the quality signal can improve execution quality without pulling the update away
from the target success level. 

\begin{figure*}[!t]
\centering
\includegraphics[width=\textwidth]{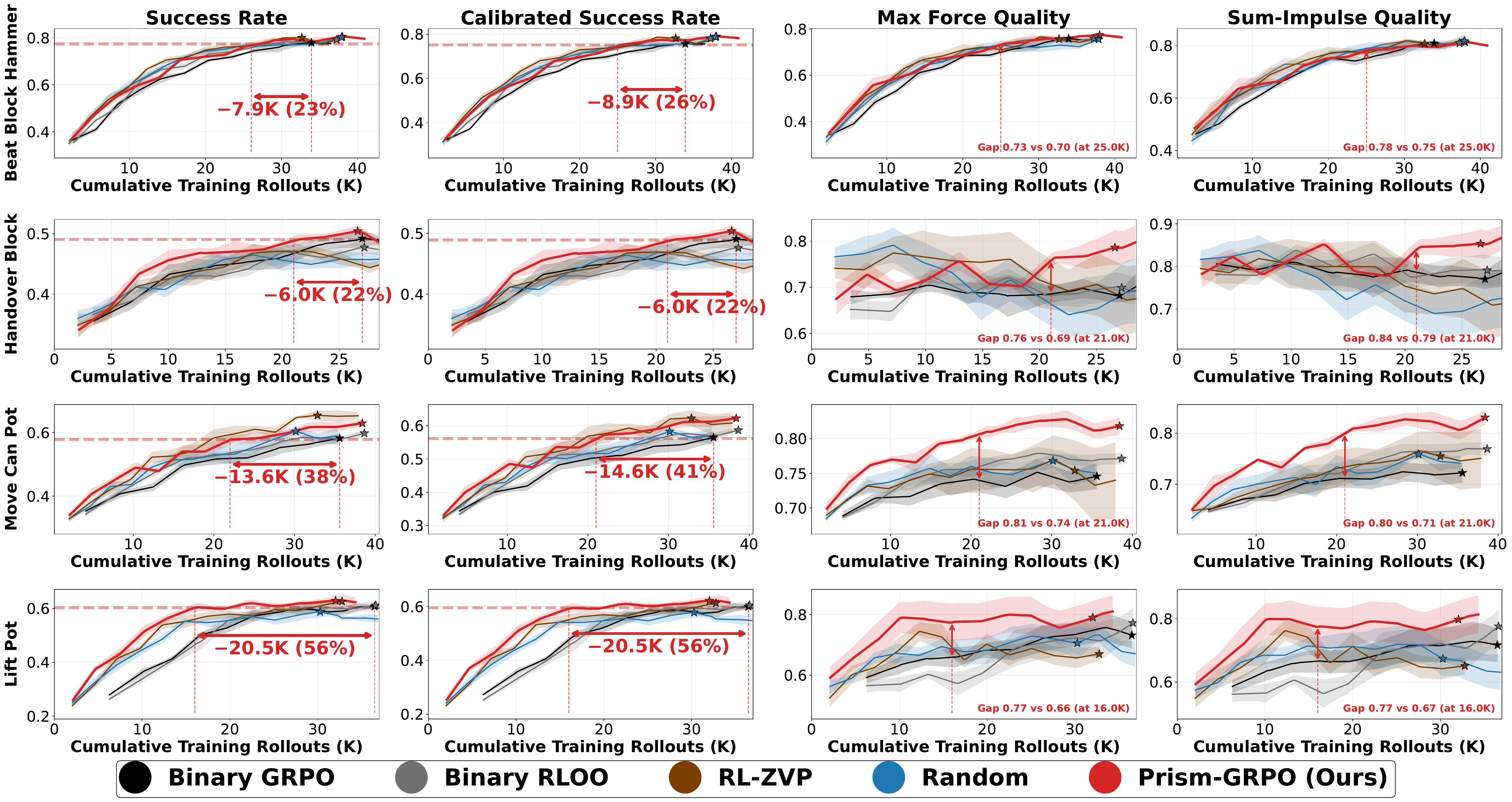}
\caption{\textbf{Main results.}$\lambda=0.2$. Stars mark each curve's peak mean success, and arrows show the rollouts saved by Prism-GRPO when matching Binary GRPO's peak mean success.}
\label{fig:main}
\end{figure*}

\subsection{Improved sample efficiency}
\label{sec:degeneracy}
Dynamic sampling discards zero-advantage groups. We measure sample efficiency by rollouts per
informative group.
% Section~\ref{sec:direction}.
\begin{theorem}[The combined reward is at least as sample-efficient as the binary reward]
\label{thm:cost}
Fix a scene with success rate $p\in(0,1)$ and group size $G\ge2$, and let $\delta(p)$ be the
probability that all $G$ combined rewards in a group coincide. The expected rollouts per
informative group are $C_{\mathrm{binary}}(p)=G/(1-p^{G}-(1-p)^{G})$ and
$C_{\mathrm{combined}}(p)=G/(1-\delta(p))$, respectively. They satisfy
\begin{equation}
\label{eq:cost}
C_{\mathrm{combined}}(p)\le C_{\mathrm{binary}}(p)
\qquad\text{for all }p\in(0,1).
\end{equation}
\end{theorem}
The proof, including coincident combined rewards, is given in
Appendix~\ref{app:cost}. For continuous conditional quality scores,
exact ties occur with probability zero, so $C_{\mathrm{combined}}(p)=G$ and every
group is informative. At $p=0.1$ or $0.9$, Binary GRPO then requires approximately
$1.76\times$ as many rollouts for $G=8$ and $2.91\times$ for $G=4$.

\paragraph{Rescued groups train quality, yet still move success.}
Theorem~\ref{thm:cost} is a fixed-scene statement: near $\bar s(x)=0$ or $\bar s(x)=1$ every
rollout in the group shares an outcome, so the binary term is constant and the group's entire
advantage spread comes from $\lambda q$. Such a rescued group thus contributes only
$g_{\mathrm{quality}}$, and appears to train quality alone. It does not. By the alignment above,
$g_{\mathrm{quality}}$ projects non-negatively onto $g_{\mathrm{success}}$, so the quality-only
gradient from a saturated scene still advances the task-level success objective, which
Proposition~\ref{prop:align} defines over all scenes rather than any one group. 
This fails only
when $g_{\mathrm{success}}=0$ over the whole task, meaning $\bar s(x)=1$ everywhere, where the task
is solved, or $\bar s(x)=0$ everywhere, where no trajectory ever succeeds. 
The zero-success case also defeats standard GRPO, which likewise assumes an SFT initialization with some task competence to bootstrap RL~\citep{li2025simplevla}; Ours requires no more.

\section{Experiments}
\paragraph{Benchmark and training setup.}
All experiments use RoboTwin 2.0~\citep{chen2025robotwin}, a bimanual manipulation benchmark
with domain randomization over clutter, lighting, background, tabletop height, and language
instruction. Each randomized task instance defines a scene
$x\sim\mathrm{Scene}(T)$, and all rollouts in a GRPO group are sampled from the same scene.
Thus, group degeneracy is determined per scene, while training aggregates updates across scenes.
We evaluate four tasks: \textit{Lift Pot}, \textit{Move Can Pot},
\textit{Handover Block}, and \textit{Beat Block Hammer}, spanning effectively single-arm and
bimanual control as well as short- to long-horizon manipulation. Task details are provided in
Appendix~\ref{app:tasks}. 
Following SimpleVLA-RL~\citep{li2025simplevla}, we use $G=8$ and 512 rollouts per update, replacing the binary reward with $\mathrm{success}+\lambda q$. We use $\lambda=0.2$ by default and ablate it in the main paper. Appendix~\ref{app:exp} provides the $G$ ablation and shows that applying the combined reward to all groups is necessary to realize the full gains; thus, we use it for all groups in the main experiments. We train each configuration with five random seeds and report means with standard-error intervals.

\paragraph{Default quality signal.}
In the main experiments, we derive $q$ from \textbf{GT-Max Force}, defined as the largest force applied to any non-target object during a rollout (Appendix~\ref{app:quality_signal}). For example, in \textit{Lift Pot}, the pot is the target object and all other objects are treated as non-target. We use this signal because unintended contacts are common, easy to measure, and can displace objects or interfere with later execution. Unless otherwise specified, the resulting quality score is used as $q$ during training. Figure~\ref{fig:ablations} evaluates two additional collision-based signals and three smoothness-based signals, showing that Prism-GRPO is not tied to a specific quality definition. 

\paragraph{Evaluation metrics.}
We evaluate both task success and execution quality. \textbf{Success rate} is the fraction of
completed trials. \textbf{Calibrated success rate} discounts successful trajectories according to
the number of non-target contacts, regardless of contact force. \textbf{Max-Force Quality}
measures the severity of the strongest non-target contact and corresponds to the default
GT-Max Force training signal. \textbf{Sum-Impulse Quality} aggregates impulse over all
non-target contacts throughout the trajectory, capturing the overall collision burden. These
metrics capture complementary aspects of collision quality. Unless otherwise specified,
calibrated success rate and Max-Force Quality are our primary success and quality metrics.
All non-binary metrics are normalized to $[0,1]$, with higher values indicating better performance.
Appendix~\ref{app:quality_metrics} provides metric definitions, while
Appendix~\ref{app:g1} reports all metrics for experiments not fully shown in the main paper.

\paragraph{Baselines.}
We compare Prism-GRPO against four baselines.
\textbf{Binary GRPO} is the SimpleVLA-RL
configuration, using the binary success reward with standard group-standardized advantages.
\textbf{Binary RLOO} replaces this estimator with leave-one-out advantages.
\textbf{Random quality} preserves the binary success term but replaces
the quality score $q$ with a random scalar, testing whether the gains come from meaningful quality
information rather than random variation. 
\textbf{RL-ZVP}~\citep{le2025no}
revives degenerate groups using entropy-shaped advantages; we adapt it to manipulation with
per-step action entropy. 
We exclude methods requiring extra supervision, such as external verifiers, self-teacher passes, or oracle demonstrations; external scoring is covered by the \textbf{VLM-Predicted} ablation.
Full details are in
Appendix~\ref{app:baseline}.

\subsection{Sample Efficiency and Execution Quality}
\label{sec:main-results}
\paragraph{Prism-GRPO improves rollout efficiency and execution quality.}
Figure~\ref{fig:main} compares all methods against total generated rollouts across four tasks. At
matched rollout budgets, Prism-GRPO improves rollout savings over Binary GRPO by $22$--$56\%$, with
calibrated success following the same trend. The largest gain occurs on \textit{Lift Pot}, where
% Binary GRPO's zero-variance filter drops up to $70\%+$ early rollouts and about $20\%$ throughout training,
% whereas Prism-GRPO keeps the discard rate near $10\%$ by recovering degenerate groups.
Binary GRPO discards ~70\% of the earliest rollout batches and stabilizes at ~20\% through the rest of training, whereas Prism-GRPO keeps the discard rate at essentially 0\% early on and around ~14\% throughout the remainder.
Appendix~\ref{app:g1} reports detailed rollout statistics for all tasks. Moreover, when
Prism-GRPO reaches a given target success rate, it also exhibits higher execution quality. 
Although Prism-GRPO is trained using Max-Force Quality,
it also improves Sum-Impulse Quality at evaluation, demonstrating cross-metric generalization rather
than overfitting to a specific collision measure.
On Handover Block, RL-ZVP shows much worse quality than the other methods:
catastrophic contacts with multiple objects can disrupt the tabletop scene, as illustrated in
Figure~\ref{fig:qualitative}(right). Overall, Prism-GRPO performs best among the baselines,
except that RL-ZVP reaches higher raw success on \textit{Move Can Pot}; however, its large gap
between raw and calibrated success shows that many successful rollouts hacked the binary checker. We explain this shortcut behavior next.

\begin{figure*}[!t]
\centering
\includegraphics[width=\textwidth]{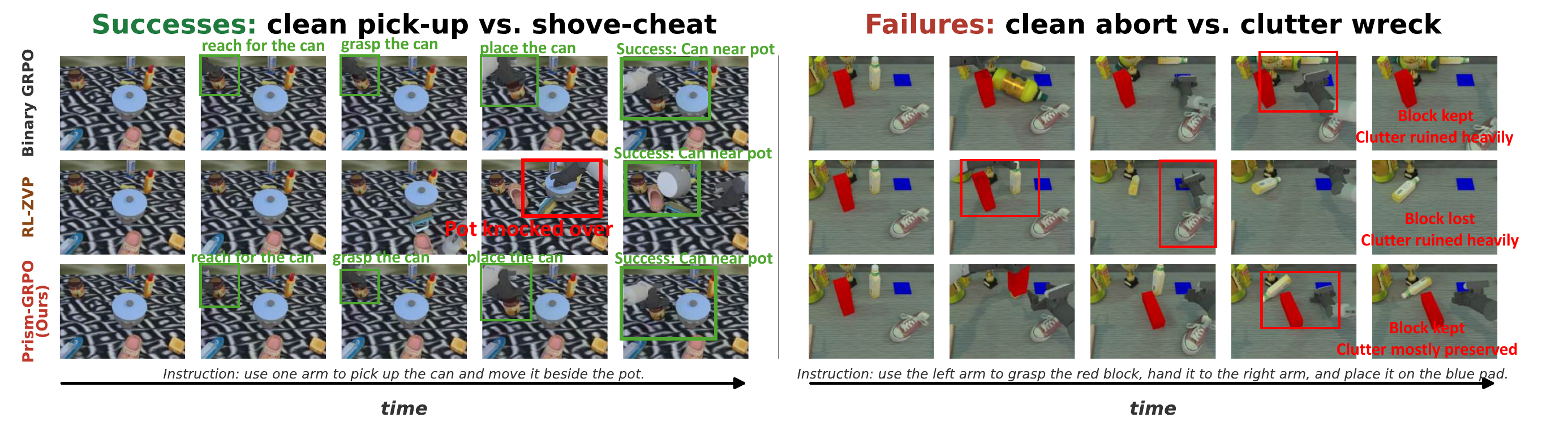}
\caption{\textbf{Qualitative examples.} 
\textbf{Left:} \textit{Move Can Pot}: RL-ZVP shove-cheat vs.\ clean executions by the other methods.
\textbf{Right:}
\textit{Handover Block} shows worse failures cause
stronger unintended contacts and larger scene disturbance.}
\label{fig:qualitative}
\end{figure*}

 \begin{figure}[!t]
    \centering
    \includegraphics[width=\linewidth]{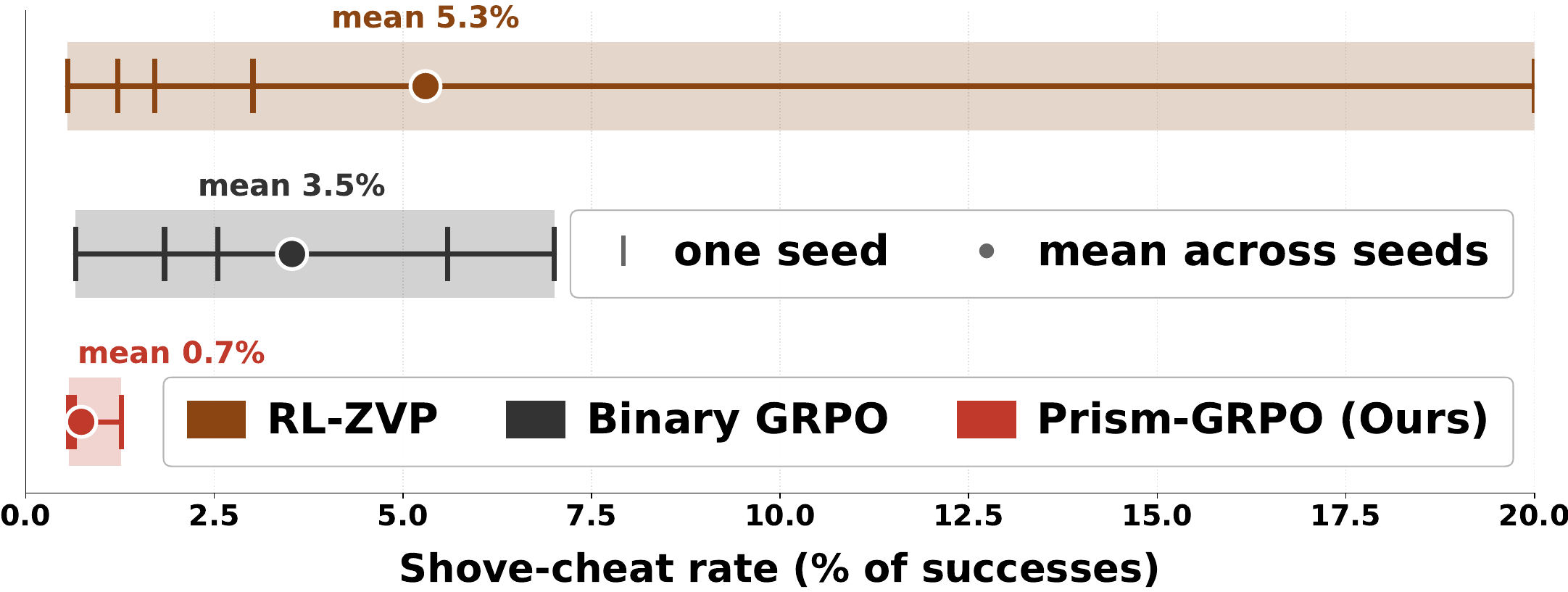}
\caption{\textbf{Shove-cheat behavior on \textit{Move Can Pot}.
}5 seeds each; ticks = per-seed rate, band = min--max, dot = mean. }
    \label{fig:shove_cheat}
\end{figure} 

\paragraph{Shove-cheat exploitation on \textit{Move Can Pot}.}
The \textit{Move Can Pot} instruction requires the arm to lift the can and
place it beside the pot, whereas the success checker considers only the
final geometry: the can must end beside the pot, regardless of whether it
was lifted or whether the pot moved. This mismatch permits a shove-cheat,
where the arm pushes the pot toward a stationary can and is still marked
successful. Figure~\ref{fig:qualitative} (left) illustrates the behavior
qualitatively: RL-ZVP strikes and displaces the pot instead of performing
the instructed lift-and-place. 
Figure~\ref{fig:shove_cheat} reports the shove-cheat rate, where a shove-cheat is defined as a successful rollout that displaces the pot by at least $10$ cm.
RL-ZVP ranges from $0.6$--$20.0\%$ and Binary GRPO from $0.7$--$7.0\%$, whereas all
Prism-GRPO seeds remain within $0.6$--$1.3\%$. The highest RL-ZVP rate is over
$15\times$ that of Prism-GRPO.

\paragraph{Quality steers online RL away from reward shortcuts.}
Online RL can discover behaviors beyond the SFT policy~\citep{li2025simplevla}, but imperfect
success checkers may also reward unintended shortcuts~\citep{skalse2022defining,yao2026multimodal}.
On \textit{Move Can Pot}, shoving the pot satisfies the geometric checker, so Binary GRPO treats
it as equivalent to a clean lift-and-place. RL-ZVP may further reinforce this shortcut through
policy confidence without assessing execution quality. Prism-GRPO instead ranks successful
trajectories by non-target contact, preserving exploration while favoring cleaner solutions. This
reduces the shove-cheat rate to nearly zero and highlights execution quality as a natural signal
available from robot trajectories.

\subsection{Sensitivity to Quality Weight}
\label{sec:ablations}
\paragraph{Meaningful quality is robust to $\lambda$, while random quality is not.}
We vary $\lambda$ to test weight sensitivity and success dominance, and compare with random quality to verify the value of meaningful signals. Results are in Table~\ref{tab:lambda}, with full curves in Appendix~\ref{app:g2}.
 Within the success-dominant regime, $\lambda<1$, Prism-GRPO
consistently reaches the target, saving $45$--$56\%$ of rollouts on \textit{Lift Pot} and
$41$--$48\%$ on \textit{Move Can Pot}, while improving quality by $11$--$15$ and $7$--$13$
points, respectively.
 Since $q\in[0,1]$, $\lambda<1$ preserves success dominance, whereas
$\lambda\ge1$ allows failures to tie or outrank successes and yields less stable performance.
The strong $\lambda=2$ result on
\textit{Lift Pot} is an outlier, as the same setting fails to reach the target on
\textit{Move Can Pot}. 
Figure~\ref{fig:main} shows that Random quality can improve success early. By making reward ties
unlikely, the random term revives all-failure groups that Binary GRPO would discard. Under GRPO
clipping bias~\citep{shao2025spurious}, these updates may reinforce favored actions and increase
success. However, as the success rate approaches $0.5$, mixed-outcome groups become common and this
tie-breaking benefit largely vanishes. Because random quality provides no success- or
quality-aligned signal, its gains are unstable and yield little execution-quality improvement.

 \begin{table}[t]
\centering
\small
\setlength{\tabcolsep}{3pt}
\begin{tabular}{llccccc|ccccc}
\toprule
& &
\multicolumn{5}{c|}{Random}
&
\multicolumn{5}{c}{Prism-GRPO} \\
\cmidrule(lr){3-7}\cmidrule(lr){8-12}
Task & Metric
& 0.2 & 0.5 & 0.9 & 1 & 2
& 0.2 & 0.5 & 0.9 & 1 & 2 \\
\midrule

\multirow{2}{*}{Lift Pot}
& CS
& \xmark & \xmark & \xmark & \xmark & \xmark
& \textbf{56} & 49 & 45 & 25 & 53 \\
& Q
& \xmark & \xmark & \xmark & \xmark & \xmark
& 11 & 14 & 15 & 14 & 26 \\

\midrule

\multirow{2}{*}{\begin{tabular}{@{}c@{}}Move \\Can Pot\end{tabular}}
& CS 
 & 24  & 33  & \xmark  &  \xmark  & \xmark 
& \textbf{41}  & 48 & 41  & 3  & \xmark \\
& Q
& 1 & 0 & \xmark  & \xmark  & \xmark 
& 7 & 10 & 13 & 16 & \xmark \\

\bottomrule
\end{tabular}
\caption{\textbf{Quality-weight $\lambda$ ablation.}
CS is the rollout saving (\%) to match Binary GRPO's calibrated success; Q is the quality gain
at that point. \(\boldsymbol{\times}\) denotes an unreached success.}
\label{tab:lambda}
\end{table}

% \paragraph{The random reward gains early but cannot sustain it.}
% Across the plots, the random reward often improves faster than the binary baselines at the
% beginning, but the gain fades and the curve does not settle at a stable high success rate.
% This happens because early training is dominated by all-failure groups: binary GRPO discards
% them, while random rewards create nonzero advantage spread and therefore let more groups
% contribute updates. This combines with the GRPO clipping bias observed by
% \citet{shao2025spurious}, where even an uninformative reward can reinforce behaviors already
% likely under the policy. As binary GRPO begins to find successes, mixed-outcome groups become
% more common and the random reward loses this rollout-efficiency advantage. Since the random
% signal is not aligned with success or quality, it provides no reliable direction after the early
% stage, so its improvement is short-lived.

\begin{figure*}[t]
    \centering
    \includegraphics[width=\linewidth]{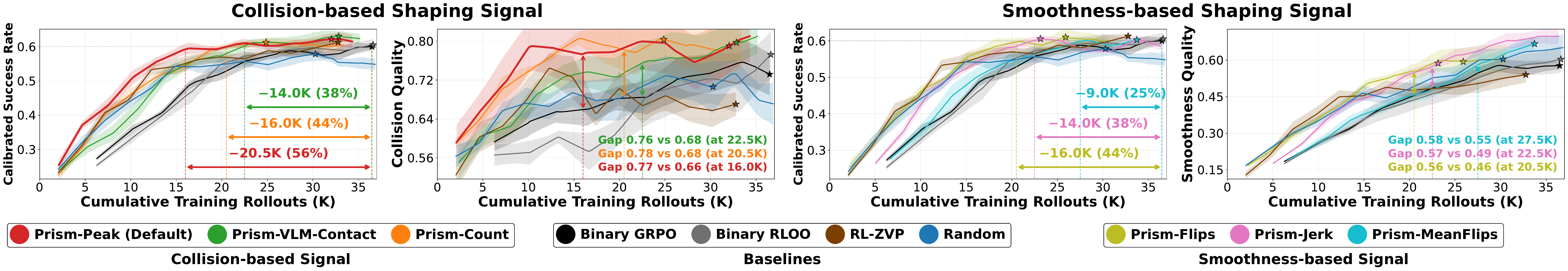}
\caption{\textbf{Quality-source ablation.} Evaluated on \textit{Lift Pot}.}
    \label{fig:ablations}
\end{figure*}

\begin{figure}[t]
    \centering
    \includegraphics[width=\linewidth]{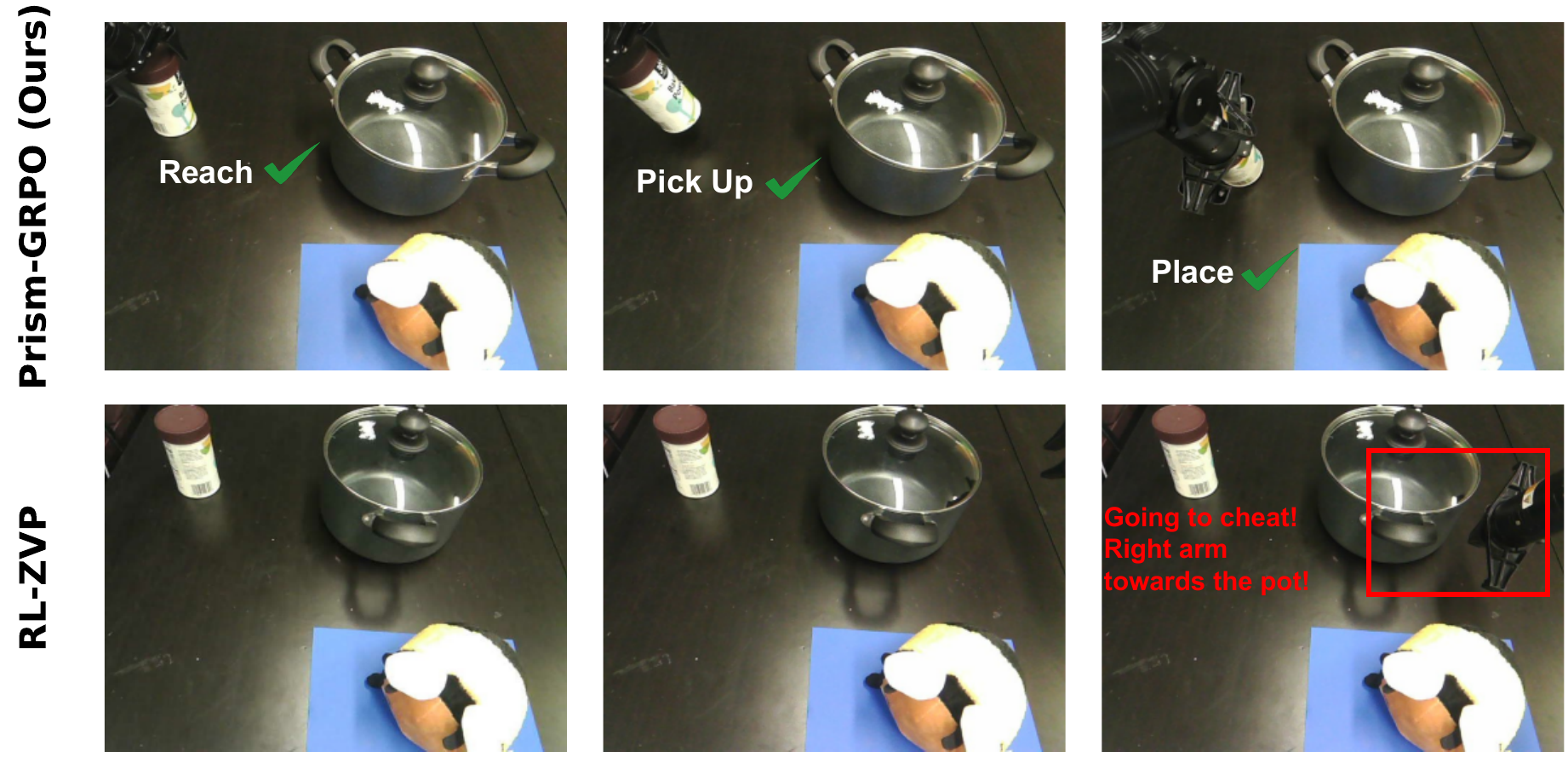}
    \caption{\textbf{Real-world behaviors.} Top: clean placement from Prism-GRPO (Ours). Bottom: shove-cheat from RL-ZVP.}
    \label{fig:real_examples}
\end{figure}

\subsection{Generalization Across Quality Sources}
We test Prism-GRPO across quality sources, treating quality as any observable trajectory signal
aligned with success.

\paragraph{Quality-source generalization.}
Figure~\ref{fig:ablations} evaluates collision cleanliness and motion smoothness using signals from
different sources. For collision quality, \textbf{Prism-Peak (Default)} uses the strongest force
applied to any non-target object, \textbf{Prism-Count} uses the number of such contacts, and
\textbf{Prism-VLM-Contact} replaces simulator contact data with a zero-shot Qwen3-VL judge to
emulate settings in which such data are unavailable. These signals save $38$--$56\%$ of rollouts
while improving collision quality. Prism-VLM-Contact yields the weakest gains among them,
consistent with its imperfect agreement with ground-truth contact labels (F1 $\approx 0.765$).
For motion quality, \textbf{Prism-Flips} counts arm-joint reversals, \textbf{Prism-Jerk} measures
the largest abrupt action change, and \textbf{Prism-MeanFlips} averages reversal counts across
action chunks. These signals save $25$--$44\%$ of rollouts while improving motion quality. See
Appendix~\ref{app:quality_signal} for details. 
Overall, Prism-GRPO is not tied to a particular metric or source: a single trajectory can expose diverse quality signals that distinguish same-outcome rollouts.

\paragraph{Quality adds little training overhead.}
Adding trajectory quality incurs little relative cost to training. Simulator- and VLM-derived quality add only 0.04\% and 4.8\% wall-clock time, respectively, Rollout generation cost remains dominant.
% Adding trajectory quality incurs little extra cost relative to training. Simulator-derived quality
% adds only $0.44$ ms per trajectory, totaling $15.9$ s, or $0.04\%$ of a 60-step run.
% VLM-Predicted is costlier at $70.7$ ms per trajectory, adding about $47$ minutes, or $4.8\%$ of
% wall-clock time, because it requires hosted inference. In both cases, rollout generation remains
% the dominant training cost.

\subsection{Real-World Evaluation}
Simulation may assign the same success label to behaviors with very different execution quality.
The shove-cheat on \textit{Move Can Pot} provides a clear test case: both clean placement and
shoving can succeed in simulation, allowing us to examine whether the higher-quality behavior
favored by Prism-GRPO transfers more reliably to hardware at comparable success.
We evaluate Binary GRPO, RL-ZVP, and Prism-GRPO on Piper over 25 trials each.

\paragraph{Execution quality improves sim-to-real reliability.}
The real-world rollouts clarify why this shortcut emerges. In simulation, the appropriate arm
depends on the can's position relative to the pot: the right arm acts when the can is on the right,
whereas the left arm should act when it is on the left. In some left-side scenes, however, the
policy reuses the learned right-arm pickup motion and strikes the intervening pot. A strong
simulated contact can displace the pot unrealistically far, allowing the shove to satisfy the
geometric success checker without executing the intended left-arm manipulation. On hardware, the
same impact produces much less displacement because of real contact dynamics and safety-limited
actuation, so these low-quality simulator successes often fail after deployment.
Figure~\ref{fig:real_examples} contrasts an RL-ZVP
shove-cheat with a clean Prism-GRPO placement. Consistent with simulation, RL-ZVP cheats most
often, Binary GRPO occasionally, and Prism-GRPO not at all while maintaining comparable clean
success (Table~\ref{tab:real_results}). Setup details are in Appendix~\ref{app:real}.

\begin{table}[t]
    \centering
    \small
    \setlength{\tabcolsep}{6pt}
    \begin{tabular}{lccc}
        \toprule
        Metric & Binary GRPO & RL-ZVP & Prism-GRPO \\
        \midrule
        Clean Success & 4/25 (16\%) & 2/25 (8\%) & \textbf{6/25 (24\%)} \\
        Shove-Cheat   & 1/25 (4\%)  & 5/25 (20\%) & \textbf{0/25 (0\%)} \\
        \bottomrule
    \end{tabular}
    \caption{\textbf{Real-world results.}
    Clean Success counts trials that move the can beside the pot without using the shove-cheat.}
    \label{tab:real_results}
\end{table}

\section{Discussion and Conclusion}
\label{sec:conclusion}

\paragraph{Positioning.}
To our knowledge, Prism-GRPO is the first VLA RL method to improve sample efficiency by recovering training signal from same-outcome groups. PolicyTrim and Z-1 instead target deployment and rollout-generation efficiency~\citep{wang2026policytrim,cao2026z}. RoboReward can break reward ties using continuous scores, but requires training an additional reward model~\citep{lee2026roboreward}. SmoothVLA improves motion quality among successful trajectories but does not target sample efficiency~\citep{li2026smoothvla}. TGRPO and GiGPO study credit assignment~\citep{chen2025tgrpo,feng2026group}, whereas RLinf-VLA and VLAJS focus on scalable training and guided exploration~\citep{zang2025rlinf,moroncelli2026vision}. These directions are complementary, but none directly recovers training signal from rollouts lost to same-outcome groups.

\paragraph{Limitations and conclusion.}
Prism-GRPO requires an observable trajectory signal aligned with success. Such signals are readily
available in robotics through contacts, actions, and vision, but may require proxies or external
verifiers in other domains. Theoretically, Prism-GRPO never increases the rollout cost of obtaining
an informative group and, when the population-level success and quality gradients align, its update
improves both objectives. Verifying this alignment directly in full VLA models remains challenging.
Overall, Prism-GRPO recovers degenerate groups while preserving success dominance, improving
sample efficiency and execution quality across tasks, signals, and real-world deployment.

% \paragraph{Limitations.}
% % Prism-GRPO requires an observable trajectory signal aligned with successful execution. Robotics
% % naturally exposes such signals through contacts, forces, actions, and visual observations, but
% % reliable quality signals may be harder to obtain in other domains. In LLM reasoning, intermediate
% % quality is often latent and must be inferred from confidence proxies, process labels, or external
% % verifiers, substantially limiting the available signal choices and sources.
% Prism-GRPO requires an observable trajectory signal aligned with success.
% Such signals are readily available in robotics via contacts, actions,
% and vision, but may require proxies or external verifiers in other domains.

% \paragraph{Conclusion.}
% We introduced Prism-GRPO, which uses trajectory quality to recover degenerate groups and improve
% sample efficiency and execution quality while preserving success dominance. Its effectiveness
% across signal types, tasks, and real-world evaluation highlights the value of trajectory evidence
% beyond final outcomes.

\section*{Acknowledgments}
Zeyun Deng at the DT Lab, Purdue University, would like to thank her advisor Ziran Wang and co-advisor Ruqi Zhang for their support in approving her summer internship. She also thanks Kai Cheng and Zhengyuan Li from Purdue University for helpful discussions. 
The authors would also like to sincerely thank the UCLA Mobility Lab for providing access to the robot platform and  providing the basic camera synchronization code used in our real-world experiments, which significantly saved our development time.

\newpage
\newpage
\bibliography{AnonymousSubmission2027}

\newpage
% appendix.tex

\appendix

\section{Proofs}

\subsection{Proof of Theorem~\ref{thm:cost}}
\label{app:cost}

\begin{proof}
Fix a scene with success rate $p\in(0,1)$ and group size $G\ge2$. Let $t_{\mathrm{succ}}$ be the
probability that all $G$ quality scores coincide given that all $G$ trajectories succeed, and let
$t_{\mathrm{fail}}$ be the analogous probability given that all $G$ fail; both lie in $[0,1]$.

Under the binary reward every trajectory scores $0$ or $1$, so a group is degenerate exactly when
all $G$ trajectories succeed or all $G$ fail. By independence this has probability
$p^{G}+(1-p)^{G}$, and a group is informative with probability $1-p^{G}-(1-p)^{G}$. Dynamic
sampling draws fresh groups until the first informative one, so the number of groups is geometric
with that success probability. At $G$ rollouts per group,
\[
C_{\mathrm{binary}}(p)=\frac{G}{1-p^{G}-(1-p)^{G}} .
\]

Under $R_{\mathrm{combined}}=\mathrm{success}+\lambda q$ with $0<\lambda<1$, the success and
failure reward ranges $[1,1+\lambda]$ and $[0,\lambda]$ are disjoint, so a mixed-outcome group can
never have all rewards equal. A combined-reward group is therefore degenerate only if all
trajectories share an outcome and, given that outcome, all $G$ quality scores coincide, giving
\begin{equation}
\label{eq:delta-def}
\delta(p)=p^{G}t_{\mathrm{succ}}+(1-p)^{G}t_{\mathrm{fail}} .
\end{equation}
The same geometric argument yields $C_{\mathrm{combined}}(p)=G/(1-\delta(p))$.

The combined cost never exceeds the binary one. Since $t_{\mathrm{succ}},t_{\mathrm{fail}}\le1$,
each term of \eqref{eq:delta-def} is at most the corresponding binary term, so
$\delta(p)\le p^{G}+(1-p)^{G}$. A smaller degeneracy probability means a larger informative
fraction $1-\delta(p)$, and hence fewer groups drawn:
$C_{\mathrm{combined}}(p)\le C_{\mathrm{binary}}(p)$, which is \eqref{eq:cost}.

The gap becomes extreme at the boundary. As $p\to1$, the binary informative fraction
$1-p^{G}-(1-p)^{G}\sim G(1-p)\to0$, so almost every group is degenerate and
$C_{\mathrm{binary}}(p)\to+\infty$; the same holds as $p\to0$. The combined cost does not blow up.
Provided quality is not almost surely constant within an outcome, so that
$t_{\mathrm{succ}},t_{\mathrm{fail}}<1$, the quality term breaks most ties, keeping
$\delta(p)\le\max(t_{\mathrm{succ}},t_{\mathrm{fail}})<1$ and $C_{\mathrm{combined}}$ finite. Its
limits are $G/(1-t_{\mathrm{succ}})$ as $p\to1$ and $G/(1-t_{\mathrm{fail}})$ as $p\to0$: at each
boundary only one kind of degenerate group remains, since an all-failure group has probability
$(1-p)^{G}\to0$ as $p\to1$ and an all-success group has probability $p^{G}\to0$ as $p\to0$.

\end{proof}

\subsection{Proof of Proposition~\ref{prop:align}}
\label{app:align}

\begin{proof}
The change in an objective along a direction $w$ is its directional derivative
$\langle\nabla_\theta(\cdot),w\rangle$. Since $\nabla_\theta J=g_{\mathrm{success}}$ and
$\nabla_\theta Q=g_{\mathrm{quality}}$, substituting
$w=g_{\mathrm{combined}}=g_{\mathrm{success}}+\lambda g_{\mathrm{quality}}$ and expanding gives
\begin{equation}
\label{eq:ascent}
\begin{aligned}
\langle \nabla_\theta J,\,g_{\mathrm{combined}}\rangle
&=\|g_{\mathrm{success}}\|^{2}
+\lambda\langle g_{\mathrm{success}},g_{\mathrm{quality}}\rangle,\\
\langle \nabla_\theta Q,\,g_{\mathrm{combined}}\rangle
&=\lambda\|g_{\mathrm{quality}}\|^{2}
+\langle g_{\mathrm{success}},g_{\mathrm{quality}}\rangle .
\end{aligned}
\end{equation}
By the alignment hypothesis, the cross term
$\langle g_{\mathrm{success}},g_{\mathrm{quality}}\rangle$ is non-negative, so the first
derivative is at least $\|g_{\mathrm{success}}\|^{2}$ and the second at least
$\lambda\|g_{\mathrm{quality}}\|^{2}$, each strictly positive when the corresponding gradient is
nonzero.
\end{proof}

\subsection{Proof of Proposition~\ref{prop:corr-to-align}}
\label{app:corr-to-align}

Throughout this appendix we abbreviate $g_S=g_{\mathrm{success}}$ and
$g_Q=g_{\mathrm{quality}}$, and for a positive definite $A$ we write
\begin{equation}
\label{eq:metric-cosine-def}
\cos_A(w,z)=\frac{w^\top A z}{\sqrt{w^\top A w}\,\sqrt{z^\top A z}}
\end{equation}
for the $A$-metric cosine between nonzero vectors $w,z$, so that $\cos_I$ is the ordinary
Euclidean cosine.

\begin{lemma}[Metric cosine under bounded conditioning]
\label{lem:metric-cosine}
Let $A\succ0$ be self-adjoint with eigenvalues in $[m,M]$ and condition number $\kappa=M/m$.
For nonzero $w,z$ with Euclidean cosine $\rho=\cos_I(w,z)$,
\begin{equation}
\label{eq:h-bound}
\begin{aligned}
\cos_A(w,z)&\ \ge\ h_\kappa(\rho),\\
h_\kappa(\rho)&=\frac{(\kappa+1)\rho-(\kappa-1)}{(\kappa+1)-(\kappa-1)\rho}.
\end{aligned}
\end{equation}
In particular, if $\rho\ge(\kappa-1)/(\kappa+1)$ then $w^\top A z\ge0$.
\end{lemma}

\begin{proof}
The left side of \eqref{eq:h-bound} depends only on $\mathrm{span}\{w,z\}$. If $\rho=\pm1$ then
$w,z$ are collinear and $\cos_A(w,z)=\rho$, so the claim holds. Assume $\rho\in(-1,1)$, making
the span two-dimensional. Each of $w^\top Az$, $w^\top Aw$, $z^\top Az$ is unchanged when $A$ is
replaced by its compression $PAP$ to this span, with $P$ the orthogonal projection onto it. By
Rayleigh--Ritz the compression has eigenvalues in $[m,M]$, hence condition number
$\kappa'\le\kappa$. Rescaling $A$ leaves $\cos_A$ unchanged, so we may take
$A=\mathrm{diag}(1,\kappa')$ on this span.

Normalize $\|w\|=\|z\|=1$. Let $e\parallel w+z$ and $f\parallel w-z$ be unit vectors; they are
orthogonal, and
\begin{equation}
\label{eq:wz-decomp}
w=ce+\sigma f,\qquad z=ce-\sigma f,
\end{equation}
with $c=\sqrt{(1+\rho)/2}$, $\sigma=\sqrt{(1-\rho)/2}$, and
$r=\sigma^2/c^2=(1-\rho)/(1+\rho)$. Writing $E=e^\top Ae$, $F=f^\top Af$, $G=e^\top Af$, a direct
calculation and $EF-G^2=\det(A)=\kappa'$ give
\begin{equation}
\label{eq:cos-EF}
\begin{aligned}
\cos_A(w,z)
&=\frac{E-rF}{\sqrt{(E+rF)^2-4rG^2}}\\
&=\frac{E-rF}{\sqrt{(E-rF)^2+4r\kappa'}} .
\end{aligned}
\end{equation}
Let $\theta$ be the angle between $e$ and the eigenvector of $A$ with eigenvalue $1$, so
$E=\cos^2\theta+\kappa'\sin^2\theta$ and $F=\sin^2\theta+\kappa'\cos^2\theta$. As $\theta$ ranges
over $[0,\pi)$ the pair $(w,z)$ ranges over all configurations with the prescribed norms and
$w^\top z=\rho$, so minimizing over $\theta$ minimizes over all such pairs. Then
\[
E-rF=(1-r\kappa')\cos^2\theta+(\kappa'-r)\sin^2\theta\ \ge\ 1-r\kappa',
\]
since $(\kappa'-r)-(1-r\kappa')=(\kappa'-1)(1+r)\ge0$. The map $t\mapsto t/\sqrt{t^2+4r\kappa'}$
is increasing on $\mathbb R$, and $(1-r\kappa')^2+4r\kappa'=(1+r\kappa')^2$, so from
\eqref{eq:cos-EF},
\begin{equation}
\label{eq:cos-lower}
\cos_A(w,z)\ \ge\ \frac{1-r\kappa'}{1+r\kappa'}=h_{\kappa'}(\rho),
\end{equation}
where the last equality follows by substituting $r=(1-\rho)/(1+\rho)$. Finally $h_\kappa(\rho)$
is decreasing in $\kappa$ for $\rho\in(-1,1)$, since
\[
\frac{\partial h_\kappa(\rho)}{\partial\kappa}
=-\frac{2(1-\rho^2)}{\big((\kappa+1)-(\kappa-1)\rho\big)^2}<0 ,
\]
so $\kappa'\le\kappa$ gives $h_{\kappa'}(\rho)\ge h_\kappa(\rho)$, proving \eqref{eq:h-bound}.
The final implication holds because the denominator of $h_\kappa(\rho)$ is positive for
$\rho<1$, $\kappa\ge1$, so $h_\kappa(\rho)\ge0$ iff $\rho\ge(\kappa-1)/(\kappa+1)$.
\end{proof}

\begin{proof}[Proof of Proposition~\ref{prop:corr-to-align}]
Fix the state $s$, write $\pi=\pi_\theta(\cdot\mid s)$ and $D=\mathrm{diag}(\pi)$, and center the
action values as $\tilde u=u-\bar u\mathbf 1$ and $\tilde v=v-\bar v\mathbf 1$, where
$\bar u=\sum_k\pi_k u_k$ and $\bar v=\sum_k\pi_k v_k$. For a single softmax decision the logit
gradients are $g_S=D\tilde u$ and $g_Q=D\tilde v$, so
$\langle g_S,g_Q\rangle=\tilde u^\top D^2\tilde v$.

Put $w=D^{1/2}\tilde u$ and $z=D^{1/2}\tilde v$. Since $\tilde u,\tilde v$ are $\pi$-centered,
$w$ and $z$ lie in the centered subspace
\begin{equation}
\label{eq:Hpi}
\mathcal H_\pi=\{y:\langle y,(\sqrt{\pi_1},\dots,\sqrt{\pi_K})\rangle=0\} ,
\end{equation}
and the policy-weighted correlation is the Euclidean cosine of $w$ and $z$, that is
$\rho(s)=\cos_I(w,z)$.

Let $P_\pi$ be the Euclidean projection onto $\mathcal H_\pi$ and set
$A_\pi=P_\pi DP_\pi|_{\mathcal H_\pi}$. Although $D$ need not preserve $\mathcal H_\pi$, for
$w,z\in\mathcal H_\pi$ we have $w^\top A_\pi z=w^\top Dz$, and therefore
\begin{equation}
\label{eq:Api-identities}
\begin{aligned}
w^\top A_\pi z &=\tilde u^\top D^2\tilde v=\langle g_S,g_Q\rangle,\\
w^\top A_\pi w &=\|g_S\|^2,
\qquad
z^\top A_\pi z =\|g_Q\|^2 .
\end{aligned}
\end{equation}
So the gradient cosine equals the $A_\pi$-metric cosine, $\cos_{A_\pi}(w,z)$, and $\kappa_\pi$ is
by definition the condition number of $A_\pi$ on $\mathcal H_\pi$. Applying
Lemma~\ref{lem:metric-cosine} with $A=A_\pi$,
\begin{equation}
\label{eq:grad-cos-bound}
\frac{\langle g_S,g_Q\rangle}{\|g_S\|\,\|g_Q\|}\ \ge\ h_{\kappa_\pi}(\rho(s)).
\end{equation}
If $\rho(s)\ge(\kappa_\pi-1)/(\kappa_\pi+1)$ then $h_{\kappa_\pi}(\rho(s))\ge0$, and since the
gradient norms are positive under the non-degeneracy assumptions, $\langle g_S,g_Q\rangle\ge0$.

For sharpness, when $\dim\mathcal H_\pi\ge2$ let $e_{\min},e_{\max}$ be orthonormal eigenvectors
of $A_\pi$ with eigenvalues $m_\pi,M_\pi$. Given $\rho\in(-1,1)$, set
\begin{equation}
\label{eq:sharp-wz}
\begin{aligned}
w &=\sqrt{\tfrac{1+\rho}{2}}\,e_{\min}+\sqrt{\tfrac{1-\rho}{2}}\,e_{\max},\\
z &=\sqrt{\tfrac{1+\rho}{2}}\,e_{\min}-\sqrt{\tfrac{1-\rho}{2}}\,e_{\max},
\end{aligned}
\end{equation}
so $\|w\|=\|z\|=1$ and $w^\top z=\rho$. Then
$w^\top A_\pi z=m_\pi\tfrac{1+\rho}{2}-M_\pi\tfrac{1-\rho}{2}$ and
$w^\top A_\pi w=z^\top A_\pi z=m_\pi\tfrac{1+\rho}{2}+M_\pi\tfrac{1-\rho}{2}$, whose ratio is
exactly $h_{\kappa_\pi}(\rho)$. Setting $\tilde u=D^{-1/2}w$ and $\tilde v=D^{-1/2}z$ gives
$\pi$-centered action values attaining the bound. Letting $\kappa_\pi\to\infty$ sends
$h_{\kappa_\pi}(\rho)\to-1$ for every $\rho<1$, so no nontrivial lower bound depending on $\rho$
alone can hold: bounded conditioning is necessary, not merely sufficient.
\end{proof}

\section{Quality}
\label{app:quality}
Prism-GRPO treats quality as a bounded trajectory-level score
$q(\tau)\in[0,1]$, with larger values indicating cleaner or smoother execution. Every Prism variant
uses the same success-dominant reward
\begin{equation}
    R(\tau)
    =
    \operatorname{success}(\tau)+\lambda q(\tau),
    \qquad
    \operatorname{success}(\tau)\in\{0,1\},
\end{equation}
and the same RLOO advantage estimator. The variants differ only in how $q(\tau)$ is measured.
This isolates the effect of the quality source while keeping the policy, success reward, and
optimization procedure fixed.

\subsection{Quality Signals}

\begin{figure}[t]
    \centering
    \includegraphics[width=\linewidth]{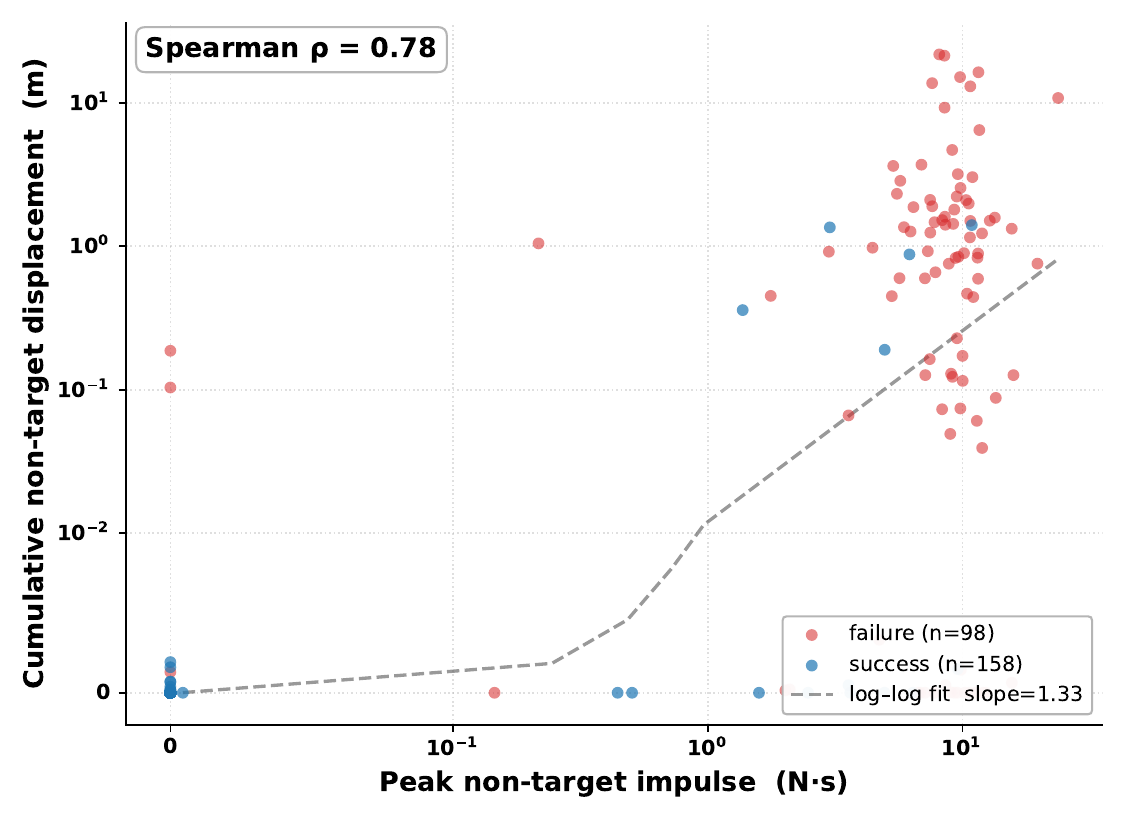}
\caption{\textbf{Peak non-target contact and object displacement on \textit{Lift Pot}.}
Stronger unintended contact is associated with larger displacement of non-target objects.}
    \label{fig:force-displacement}
\end{figure}
\label{app:quality_signal}
\paragraph{Collision signals.}
We seek task-agnostic collision signals that describe how cleanly a trajectory is executed rather
than how far the task has progressed. For each task, we designate the manipulated object as the
target object. For example, the target is the pot in \textit{Lift Pot} and the can in
\textit{Move Can Pot}. Contact with the target is required for manipulation and is therefore
excluded, while contact between the robot and any other object is treated as non-target contact.

\paragraph{Simulator-based collision signals.}
When simulator contact data are available, we derive two trajectory-level signals directly from
the contact log. Let $\mathcal{C}_{\mathrm{nt}}(\tau)$ denote the set of non-target contact events
in trajectory $\tau$, and let $I(c)$ denote the impulse of event $c$. We define
\begin{equation}
    I_{\mathrm{peak}}(\tau)
    =
    \max_{c\in\mathcal{C}_{\mathrm{nt}}(\tau)} I(c),
    \qquad
    N_{\mathrm{nt}}(\tau)
    =
    \left|\mathcal{C}_{\mathrm{nt}}(\tau)\right|,
\end{equation}
with both quantities set to zero when no non-target contact occurs.
\textbf{Prism-Peak (Default)}, referred to as \textbf{GT-Max Force} in the main paper, uses
$I_{\mathrm{peak}}(\tau)$ to capture collision severity. Stronger unintended contact is more likely
to displace surrounding objects and affect subsequent execution, as illustrated in
Figure~\ref{fig:force-displacement}. \textbf{Prism-Count} instead uses
$N_{\mathrm{nt}}(\tau)$ to capture collision frequency, assigning equal weight to each contact
event regardless of its magnitude. 

\paragraph{Prism-VLM-Contact.}
When simulator contact logs are unavailable, \textbf{Prism-VLM-Contact} estimates collisions
directly from visual observations. A zero-shot Qwen3-VL-235B judge receives eight temporally
ordered frames and determines whether the robot contacts any object other than the task target.
The prompt identifies the target, excludes its intended manipulation, and asks the judge to inspect
both changes in the scene layout and brief contacts in intermediate frames. We convert its output
into a binary quality signal indicating whether the trajectory is collision-free. On $256$
trajectories, the judge achieves $78.9\%$ accuracy and a contact-class F1 score of $69.0\%$,
with $90.9\%$ precision and $55.6\%$ recall. It correctly recognizes $142$ of $148$
collision-free trajectories but detects only $60$ of $108$ trajectories containing contact,
showing that it provides a conservative visual alternative to simulator-derived contact signals.

\paragraph{Full Lift Pot collision-judging prompt.}
\label{app:vlm_prompt_lift_pot}

\begin{quote}
\footnotesize
\raggedright

\textbf{Decision rule.}

Decide whether the robot \textbf{COLLIDED} with anything other than its target object at any point
during the trajectory.

\medskip
\noindent
In this task the target object is: \textbf{THE POT}.

\medskip
\noindent
Touching, gripping, closing on, lifting, holding, or moving \textbf{THE POT} is the task itself---it
is \textbf{NOT} a collision, no matter how firm. The pot rising in the air, the gripper closing
around the pot, and the pot leaving the table are all expected and are not counted.

\medskip
\noindent
\textbf{What counts as a collision:}
\begin{itemize}
    \item The robot arm or gripper visibly contacts the table, including pressing into it, sliding
    along it, or pushing against it.
    \item The robot contacts the wall behind it.
    \item The robot bumps, knocks, slides, tips, or pushes any distractor object that is not the pot.
    \item Any non-pot object visibly moves because of the robot.
    \item Even a brief, light touch that displaces a non-target object counts; there is no
    ``too gentle to count'' threshold.
\end{itemize}

\noindent
\textbf{How to read the images:}
\begin{itemize}
    \item You are given $N$ frames sampled in temporal order from a single trajectory. The first
    image is the start of the trajectory, before the robot moves, and the last image is the end,
    after the robot has stopped. The middle frames are evenly spaced snapshots in between.
    \item The frames are not independent. Read them as a short film. Anchor the judgment on the
    first frame, which shows the initial layout, and the last frame, which shows the final layout.
    Use the middle frames to identify contacts or displacements that occur only briefly.
\end{itemize}

\noindent
\textbf{How to look for collisions across frames:}
\begin{enumerate}
    \item Identify the pot and ignore its motion entirely because it is supposed to move.
    \item In the first frame, identify every other visible object, including the table surface,
    wall, and distractors. Note their positions, orientations, and stacking configurations as the
    reference state.
    \item Compare the last frame with the first frame for each non-target object:
    \begin{itemize}
        \item Has it shifted left, right, forward, or backward?
        \item Has it rotated, tipped, or fallen?
        \item Has it been knocked off the table or scattered?
    \end{itemize}
    Any visible displacement of a non-target object between the first and last frames indicates
    that a collision occurred during the trajectory.
    \item Inspect the middle frames in temporal order and follow the robot arm and gripper path.
    Determine whether they visibly intersect the table surface, wall, or space occupied by a
    distractor. Contact in a middle frame counts even if the non-target object returns to its
    original position by the end.
    \item Look for indirect evidence, including motion blur on a distractor, the gripper
    momentarily pressed flat against the table, or a distractor's shadow moving as the robot passes
    nearby.
\end{enumerate}

\noindent
\textbf{Anchors:}
\begin{itemize}
    \item[2 ---] \textbf{No collision:} Comparing the last frame with the first frame, all non-pot
    objects remain in the same positions and orientations, and the robot never visibly contacts
    the table, wall, or a distractor.
    \item[1 ---] \textbf{Collision:} At least one non-pot object visibly shifts, rotates, tips, or
    moves between the first and last frames, or the robot visibly contacts the table, wall, or a
    distractor in any frame.
\end{itemize}

\noindent
\textbf{Task, for context only---do not score task success:}

Use both arms to lift the pot.

\medskip
\noindent
Respond with the digit first, followed by a brief reason. Use exactly the following format:

\begin{quote}
\ttfamily
ANSWER: <1|2>\\
REASON: <one short sentence naming the specific frame(s) and object(s) that drove your decision>
\end{quote}

\end{quote}

\paragraph{Smoothness signals.}
We derive three task-agnostic smoothness signals from the executed action sequence, without
requiring simulator contact data. For the $14$-dimensional ALOHA action space, we use the twelve
arm-joint dimensions and exclude the two gripper dimensions, so intended opening and closing
motions are not treated as irregular behavior. Let $\mathcal{A}$ denote the set of arm joints,
$K_\tau$ the number of action chunks in trajectory $\tau$, and $f_j(\tau)$ the number of
velocity-direction reversals of joint $j$. Each signal captures a different form of irregular
motion.

\paragraph{Prism-Flips.}
This  captures localized oscillation, where a single unstable joint repeatedly
changes direction even when the remaining joints move smoothly. It uses the largest reversal count
among the arm joints:
\begin{equation}
    s_{\mathrm{flips}}(\tau)
    =
    \frac{1}{K_\tau}
    \max_{j\in\mathcal{A}} f_j(\tau).
    \label{eq:signal_flips}
\end{equation}
The signal is therefore determined by the most oscillatory joint in the trajectory.

\paragraph{Prism-MeanFlips.}
This captures reversals distributed across the entire arm rather than focusing
on a single joint. It averages the reversal count over all arm joints:
\begin{equation}
    s_{\mathrm{meanflips}}(\tau)
    =
    \frac{1}{K_\tau |\mathcal{A}|}
    \sum_{j\in\mathcal{A}} f_j(\tau).
    \label{eq:signal_meanflips}
\end{equation}
This signal reflects the overall level of oscillatory motion and is less sensitive to one isolated
noisy joint.

\begin{figure*}[t]
    \centering
    \includegraphics[width=\textwidth]{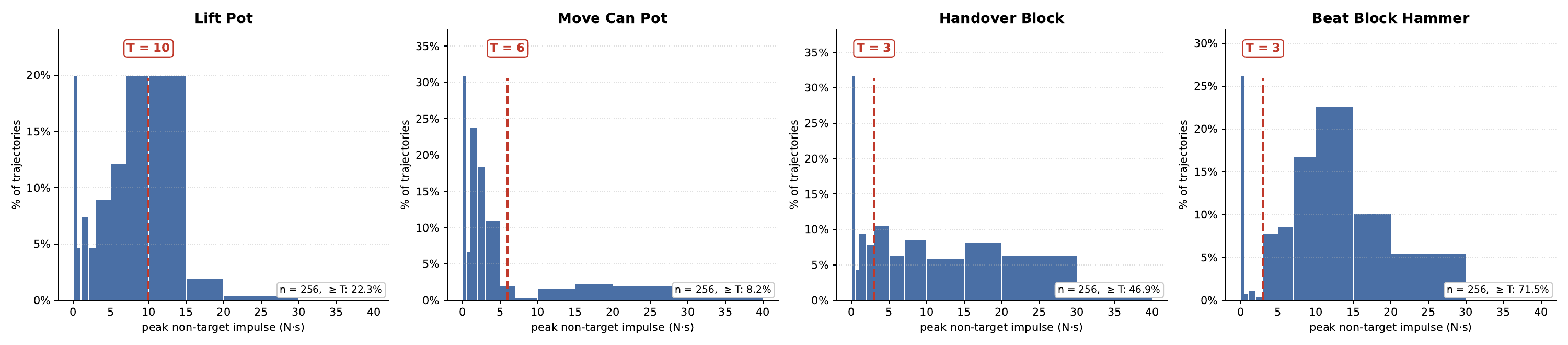}
    \caption{\textbf{Collision-threshold calibration.}
   Dashed lines mark the
    thresholds selected at the boundary between the low-impulse cluster and the high-impulse tail.}
    \label{fig:t_choice_collision}
\end{figure*}

\begin{figure*}[t]
    \centering
    \includegraphics[width=\textwidth]{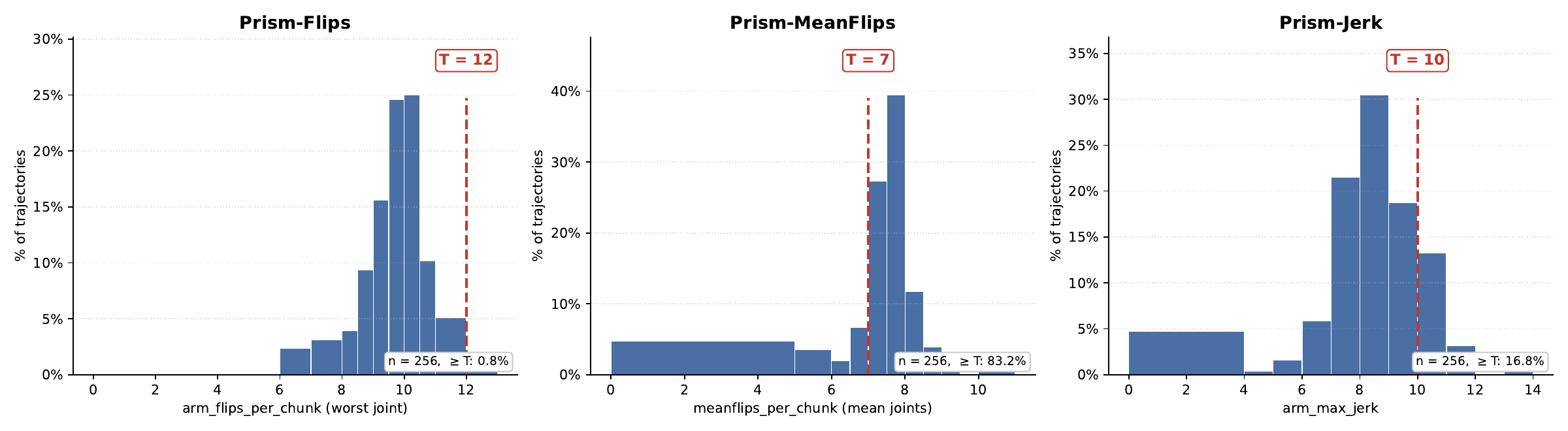}
    \caption{\textbf{Smoothness-threshold calibration on \textit{Lift Pot}.}}
    \label{fig:t_choice_smoothness}
\end{figure*}

\paragraph{Prism-Jerk.}
This captures abrupt changes in commanded arm motion rather than repeated
direction reversals. Let $\mathbf{a}^{\mathrm{arm}}_t$ denote the arm-joint action vector at
timestep $t$, and let $T_\tau$ denote the number of evaluated timesteps. We compute the average
magnitude of the third-order action difference:
\begin{equation}
    s_{\mathrm{jerk}}(\tau)
    =
    \frac{1}{T_\tau}
    \sum_{t=1}^{T_\tau}
    \left\lVert
        \Delta^3 \mathbf{a}^{\mathrm{arm}}_t
    \right\rVert_2.
    \label{eq:signal_jerk}
\end{equation}
Thus, Prism-Flips measures the most severe single-joint oscillation, Prism-MeanFlips measures
oscillation across the arm, and Prism-Jerk measures the average magnitude of abrupt action
changes. We evaluate the three signals separately.

\subsection{Metrics}
\label{app:quality_metrics}

\paragraph{Quality signals and normalization.}
We consider six trajectory-level quality signals. For collision,
\textbf{Prism-Peak (Default)}, referred to as \textbf{GT-Max Force} in the main paper, uses the
peak non-target contact impulse; \textbf{Prism-Count} uses the number of non-target contact events;
and \textbf{Prism-VLM-Contact} uses a zero-shot VLM to estimate non-target contacts from sampled
trajectory frames. For smoothness, \textbf{Prism-Flips} uses the largest velocity-sign reversal
count among the arm joints per action chunk, \textbf{Prism-MeanFlips} averages the reversal count
across all arm joints, and \textbf{Prism-Jerk} measures action jerk. 
We convert each non-negative raw cost $r(\tau)$ into a quality score in $[0,1]$:
\begin{equation}
    q(\tau)
    =
    \max\left(
        0,\,
        1-
        \frac{\max\!\left(0,r(\tau)-r_0\right)}{T}
    \right),
    \label{eq:quality_norm}
\end{equation}
where larger values indicate cleaner or smoother execution. The threshold $T$ determines the
normalization scale, while $r_0$ is an optional floor for motion that is inherently required by the
task. Collision signals always use $r_0=0$. For Prism-Flips on \textit{Move Can Pot} and
\textit{Beat Block Hammer}, we use $r_0=12$, estimated from the median flip count of successful
SFT trajectories. This removes the reversals naturally required by grasp-and-place and
lift-and-strike motions before measuring excess oscillation.

\paragraph{Calibration from the SFT policy.}
Before reinforcement learning, we begin with the SFT policy. For each task, we run this policy on
the same set of $256$ validation scenes and examine the resulting distribution of each raw quality
signal. We then choose $T$ from this distribution so that ordinary SFT trajectories retain useful
variation in quality, while clearly excessive collision or motion values receive low scores. Once
selected, the threshold is fixed for all subsequent RL runs and is not adjusted using RL
performance, individual seeds, or reported checkpoints.

\paragraph{Collision thresholds.}
For Prism-Peak, we examine the SFT distribution of peak non-target impulse for each task and place
$T_{\mathrm{peak}}$ at the boundary between the low-impulse cluster and the high-impulse tail.
The low-impulse region mainly contains gentle grazing or placement contacts, while the tail
corresponds to more severe unintended interactions such as arm slams or object tipping. This choice
preserves quality differences among relatively clean trajectories while assigning zero quality to
clearly excessive contacts. As shown in Figure~\ref{fig:t_choice_collision}, the resulting
thresholds are $10\,\mathrm{N\,s}$ for \textit{Lift Pot},
$6\,\mathrm{N\,s}$ for \textit{Move Can Pot}, and $3\,\mathrm{N\,s}$ for both
\textit{Handover Block} and \textit{Beat Block Hammer}. The distributions are computed after
applying the same target exemptions and grace filters used during training, including exempting
task-required placement contact in \textit{Handover Block} and contact with the manipulated hammer
in \textit{Beat Block Hammer}. For Prism-Count and cumulative-impulse quality, the SFT
distributions have comparable scales across tasks, so we use the shared thresholds $T_n=30$ and
$T_{\Sigma}=30$.

\paragraph{Smoothness thresholds.}
Figure~\ref{fig:t_choice_smoothness} shows the SFT distributions used to calibrate the three
smoothness thresholds on \textit{Lift Pot}. For Prism-Flips, we set
$T_{\mathrm{flips}}=12$ near the upper tail because the signal is determined by the most
oscillatory arm joint. For Prism-Jerk, we similarly set $T_J=10$ near the upper tail so that
ordinary variation in commanded motion retains positive quality while unusually abrupt motion is
penalized. Prism-MeanFlips instead averages reversals across all arm joints and has a more
concentrated distribution around typical behavior. Placing its threshold at the upper tail would
assign nearly all trajectories high quality and provide little distinction among them. We
therefore set $T_{\mathrm{meanflips}}=7$ near the SFT median, preserving useful variation in the
quality score. All three thresholds are selected before RL and remain fixed during training and
evaluation.

\paragraph{Robustness to threshold choice.}
The SFT-based calibration above provides a principled default for selecting $T$, but the method
should not depend critically on the exact placement of this threshold. We therefore test
Prism-Peak on \textit{Lift Pot} using both the calibrated value
$T_{\mathrm{peak}}=10\,\mathrm{N\,s}$ and a substantially smaller value,
$T_{\mathrm{peak}}=3\,\mathrm{N\,s}$. The latter changes the normalization scale by more than a
factor of three and clips a much larger portion of the impulse distribution to zero. Despite this
large change, Figure~\ref{fig:quality_threshold} shows similar learning behavior and rollout
savings under both settings. This result suggests that the calibration rule provides a reasonable
default, while Prism-GRPO remains effective over a broad range of threshold values and does not
require precise tuning of $T$.

\begin{figure*}[t]
    \centering
    \includegraphics[width=\textwidth]
    {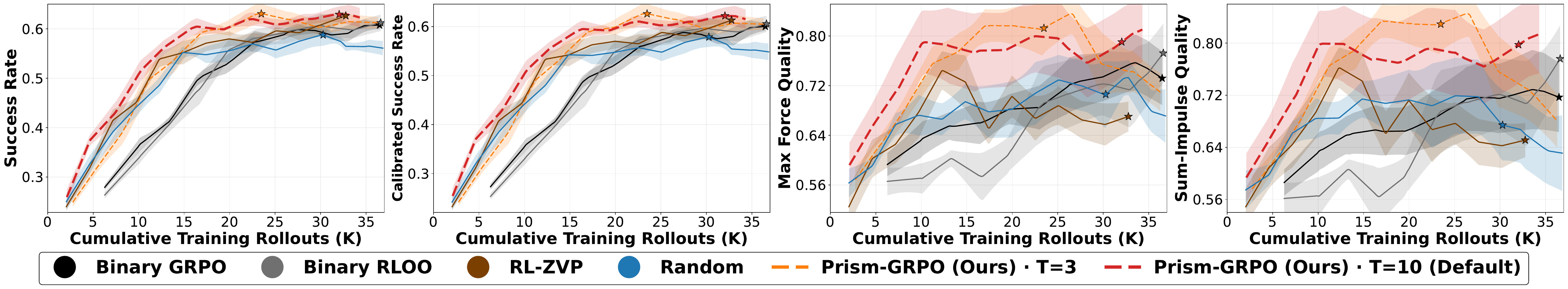}
    \caption{\textbf{Sensitivity to the quality threshold on \textit{Lift Pot}.}
    Prism-Peak shows similar performance for
    $T_{\mathrm{peak}}\in\{3,10\}\,\mathrm{N\,s}$, with
    $10\,\mathrm{N\,s}$ used as the default.}
    \label{fig:quality_threshold}
\end{figure*}

\section{Experimental Setup}
\label{app:setup}
The following configuration is shared by Prism-GRPO and all baselines.

\paragraph{Initialization.}
We build directly on SimpleVLA-RL~\citep{li2025simplevla}, whose full code and supervised
fine-tuning (SFT) checkpoints are publicly released. Every run in this paper, including Prism-GRPO
and all baselines, starts from the same released SFT checkpoint for the corresponding task, so the
methods differ only in the reinforcement-learning stage rather than initialization.

\paragraph{Policy and action space.}
The policy is OpenVLA-OFT with a discrete action head. Each degree of freedom is uniformly
quantized into $256$ bins, and one action token is emitted per degree of freedom. On RoboTwin
with the ALOHA embodiment, each action chunk spans $25$ steps with $14$ action tokens per step.

\paragraph{Training configuration.}
Each RL step samples $64$ scenes with group size $G=8$, yielding $512$ rollouts per step. We use $G=8$ as the default setting throughout the main paper, following
the standard SimpleVLA-RL configuration; its effect is studied separately in
Section~\ref{app:group-size}. We use a learning rate of $5\times10^{-6}$ with constant warmup,
gradient clipping at $1$, PPO clip bounds $(0.2,0.28)$, no KL penalty or entropy bonus, and a
rollout temperature of $1.6$.

\begin{algorithm}[t]
\caption{Adaptive gap-fill rollout generation}
\label{alg:adaptive_gap_fill}
\begin{algorithmic}[1]
\REQUIRE Scene batch size $B$, rollouts per scene $G$, data-parallel size $D$
\STATE $\mathcal{V} \gets \varnothing$
\STATE $N_{\mathrm{gen}} \gets 0$
\WHILE{$|\mathcal{V}| < BG$}
    \IF{$N_{\mathrm{gen}} = 0$}
        \STATE $B_{\mathrm{req}} \gets B$
    \ELSE
        \STATE $N_{\mathrm{keep}} \gets |\mathcal{V}|$
        \STATE $r_{\mathrm{keep}} \gets
        \max\!\left(0.3,\frac{N_{\mathrm{keep}}}{N_{\mathrm{gen}}}\right)$
        \STATE $N_{\mathrm{short}} \gets BG-N_{\mathrm{keep}}$
        \STATE $B_{\mathrm{short}} \gets
        \left\lceil\frac{N_{\mathrm{short}}}{G}\right\rceil$
        \STATE $\widetilde{B} \gets
        \left\lceil\frac{1.2\,B_{\mathrm{short}}}{r_{\mathrm{keep}}}\right\rceil$
        \STATE $\widetilde{B} \gets
        \max\!\left(D,\operatorname{RoundUp}(\widetilde{B},D)\right)$
        \STATE $B_{\mathrm{req}} \gets \min(B,\widetilde{B})$
    \ENDIF
    \STATE Generate $B_{\mathrm{req}}G$ trajectories from fresh scenes
    \STATE $N_{\mathrm{gen}} \gets N_{\mathrm{gen}}+B_{\mathrm{req}}G$
    \STATE Apply dynamic filtering and append surviving trajectories to $\mathcal{V}$
\ENDWHILE
\STATE Retain the first $BG$ trajectories in $\mathcal{V}$ for the policy update
\end{algorithmic}
\end{algorithm}

\begin{figure*}[!t]
    \centering
    \includegraphics[width=\textwidth]{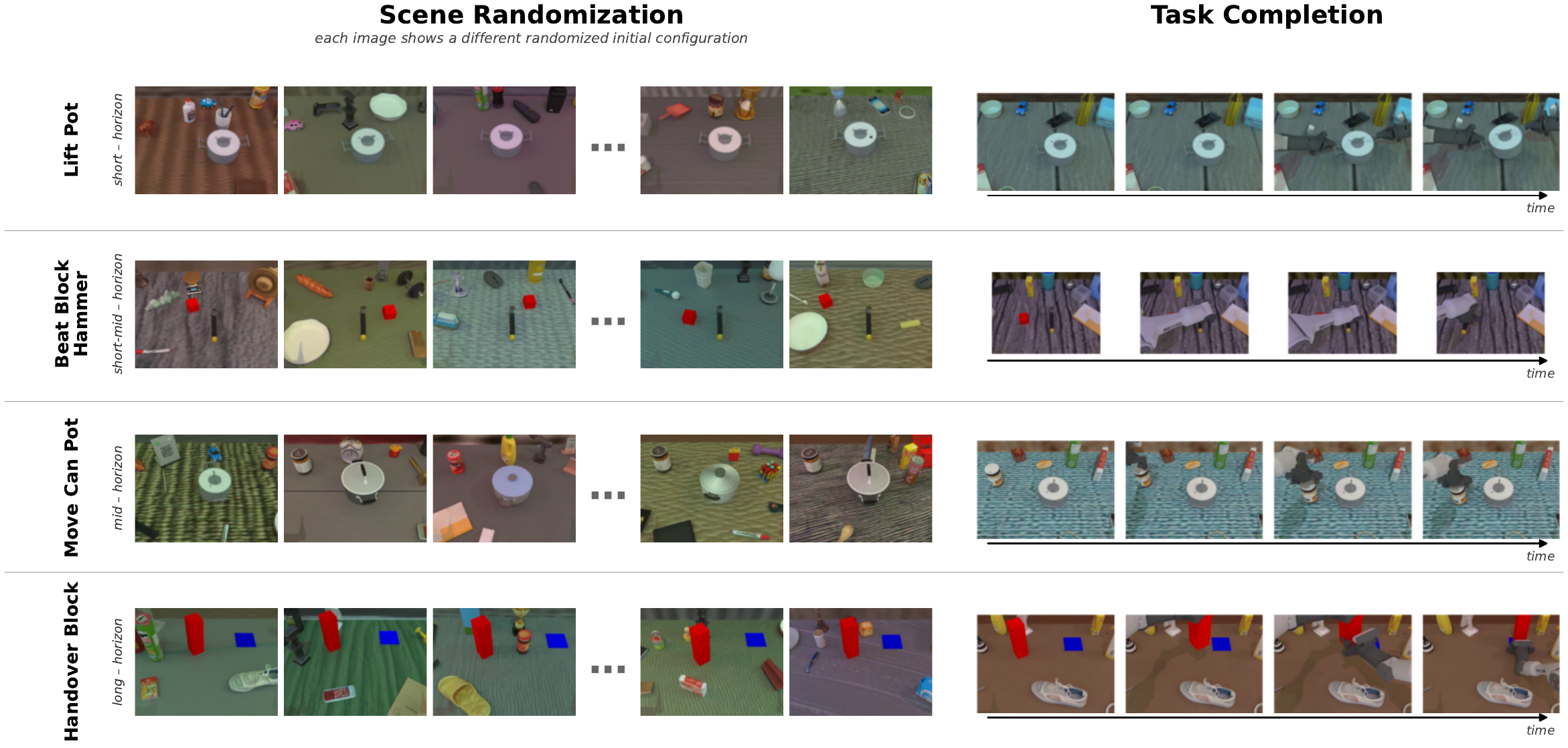}
    \caption{\textbf{Task overview.} Four tasks under randomized object placements, clutter,
    lighting, and backgrounds.}
    \label{fig:tasks}
\end{figure*}

\paragraph{Avoiding unnecessary full-batch refills.}
The original SimpleVLA-RL dynamic-sampling procedure refills the retained batch in fixed increments
of $64$ scenes, or $512$ trajectories. After the initial batch is generated and degenerate groups
are removed, it requests another full batch whenever the retained set falls below the target size,
regardless of the remaining deficit. This can generate substantial surplus. For example, with a
$95\%$ keep rate, only about $26$ additional trajectories are needed, yet the procedure still
generates another $512$.

\paragraph{Adaptive gap-fill rollout generation.}
We replace the fixed-size refill with an adaptive strategy while keeping the optimizer batch
unchanged. Each update targets $B=64$ retained scenes with $G=8$ rollouts per scene, for a total of
$512$ retained trajectories. After each filtering round, we estimate the number of additional
scenes required from the observed keep rate, lower-bound the keep-rate estimate by $0.3$, add
$20\%$ headroom, round the request to a multiple of the data-parallel group size, and cap it at
$B$ scenes. Refill rounds use fresh scenes and continue until the target batch is complete.
Algorithm~\ref{alg:adaptive_gap_fill} gives the full procedure. All methods therefore update on
the same number of retained trajectories, while every generated trajectory, including filtered
and surplus trajectories, is counted toward rollout cost.

\subsection{Tasks}
\label{app:tasks}

\paragraph{Task selection.}
We evaluate on four RoboTwin tasks that cover different manipulation skills and coordination
patterns. \textit{Lift Pot} and \textit{Handover Block} require bimanual coordination, while
\textit{Move Can Pot} and \textit{Beat Block Hammer} primarily use one arm. Together, the tasks
include lifting, pick-and-place, handover, and tool use, and expose different failure modes such as
unstable grasps, object tipping, inaccurate placement, failed transfer, and reward shortcuts.

\paragraph{\textit{Lift Pot}.}
The robot uses both arms to grasp opposite sides of a pot and lift it while keeping the pot upright.
The pot variant and scene configuration are randomized. Success requires a stable bimanual grasp,
sufficient lifting height, and limited pot tilt. Common failures include one-sided grasps, gripper
slip, and tipping during the lift.

\paragraph{\textit{Beat Block Hammer}.}
The robot grasps a hammer with the arm closest to the block and strikes the block with the hammer
head. The block pose is randomized across scenes. Success requires contact between the hammer head
and the block with accurate horizontal alignment. Common failures include incorrect hammer
orientation, poor strike alignment, and missing the block.

\paragraph{\textit{Move Can Pot}.}
The robot picks up a can and places it beside a pot. The can pose, object variants, and instructed
arm are randomized, while the pot remains near the center of the workspace. Success requires the
can to be upright, released, and placed within the target region beside the pot. A characteristic
failure is the shove-cheat, in which the policy moves the pot toward the stationary can instead of
grasping and moving the can.

\paragraph{\textit{Handover Block}.}
The left arm grasps a block, transfers it to the right arm, and places it on a target pad. The block
and pad poses are randomized. Success requires completing the handover and releasing the block at
the target location. Common failures include dropping the block during transfer, incomplete
handover, and inaccurate final placement.

\begin{figure*}[t]
    \centering
    \includegraphics[width=\textwidth]{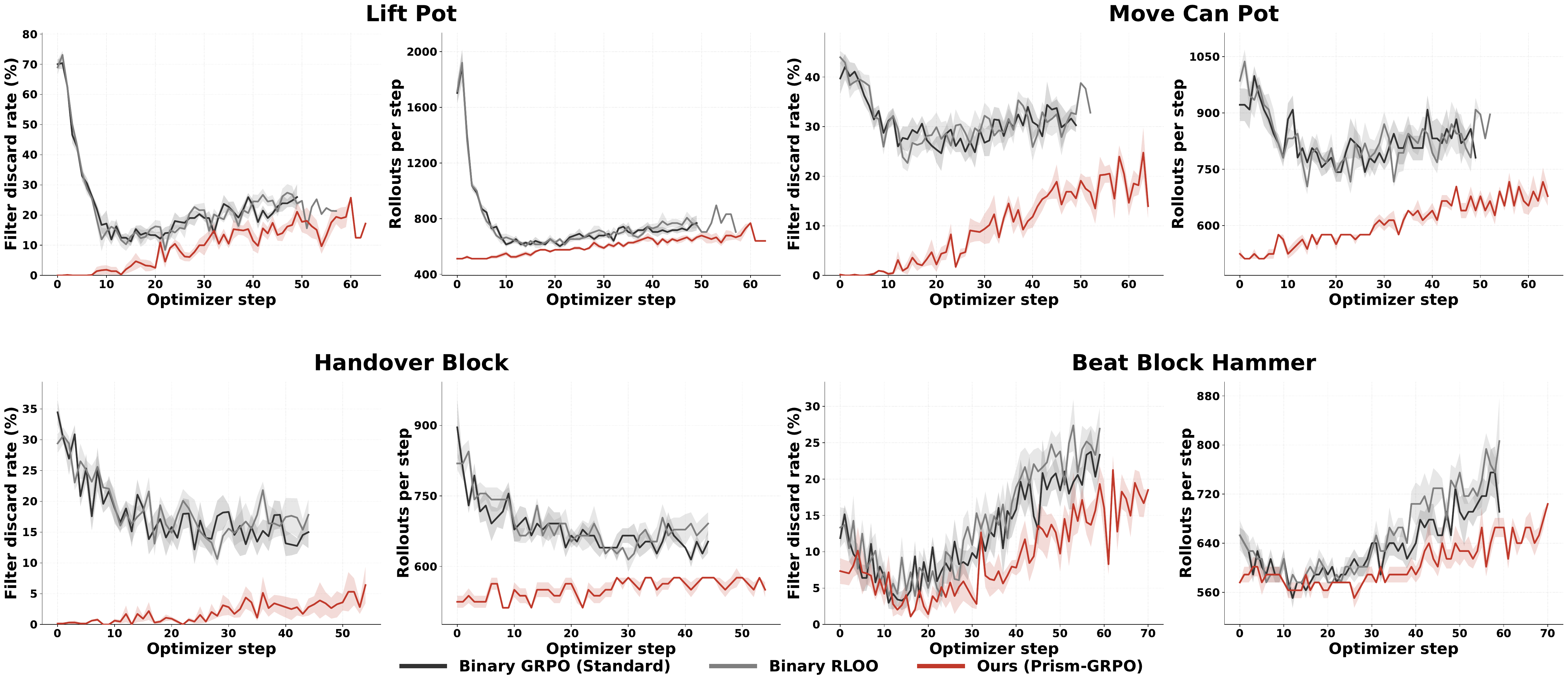}
    \caption{\textbf{Dynamic-filtering statistics across four tasks.}
    Each task panel reports the percentage of zero-variance groups discarded by the trainer
    (left) and the total number of trajectories generated per optimizer step (right).
    Prism-GRPO consistently discards fewer groups and consequently remains closer to the minimum
    cost of $512$ rollouts per step.}
    \label{fig:filter_statistics}
\end{figure*}

\section{Advantage Calculation}
\label{app:advantage}

This section describes how Prism-GRPO converts the combined reward into advantages and why the
leave-one-out estimator provides learning signal for same-outcome groups whenever their trajectory
qualities differ.

\paragraph{Combined reward.}
Each trajectory receives
$R(\tau)=\mathrm{success}(\tau)+\lambda q(\tau)$ with $\lambda=0.2$, where
$q(\tau)\in[0,1]$ is the quality score defined in Section~\ref{app:quality_metrics}. The quality
term therefore lies in $[0,\lambda]$. Any successful trajectory has reward at least $1$, while any
failed trajectory has reward at most $\lambda$, so every success continues to outrank every failure.
The quality term only distinguishes trajectories within the same outcome class and is applied to
mixed- and same-outcome groups alike.

\paragraph{Leave-one-out advantage.}
We use the leave-one-out estimator, which compares each trajectory with the mean reward of the
remaining group members:
\begin{equation}
    A_i
    =
    R_i-\frac{1}{G-1}\sum_{j\neq i}R_j.
    \label{eq:app-rloo}
\end{equation}
Unlike standard group-normalized GRPO, RLOO does not divide by the within-group reward standard
deviation. In a same-outcome group, every trajectory shares the same binary outcome $b$, so
$R_i=b+\lambda q_i$. The constant outcome term cancels:
\begin{equation}
\begin{aligned}
    A_i
    &=
    (b+\lambda q_i)
    -
    \left(
        b+\frac{\lambda}{G-1}\sum_{j\neq i}q_j
    \right) \\
    &=
    \lambda
    \left(
        q_i-\frac{1}{G-1}\sum_{j\neq i}q_j
    \right).
\end{aligned}
\label{eq:app-cancel}
\end{equation}
Thus, whenever quality varies within a same-outcome group, the group receives nonzero advantages
and contributes learning signal where a binary reward would give none. This behavior requires no
special handling in the implementation; the same estimator in Equation~\eqref{eq:app-rloo} is
applied to every group.
Standard group normalization would retain the ordering induced by quality, but it would largely
cancel the scale of $\lambda$ because the reward standard deviation in a same-outcome group is also
proportional to $\lambda$. RLOO instead preserves the intended magnitude of the quality term, which
is why we use it throughout Prism-GRPO. We empirically ablate this design choice in
Section~\ref{app:rloo-lambda}, where we compare RLOO with the standard group-normalized GRPO
estimator.

\section{Baseline Implementation}
\label{app:baseline}

All baselines use the setup of Appendix~\ref{app:setup} and differ from Prism-GRPO only in the
reward and the advantage estimator.

\paragraph{Binary GRPO.}
The reward is the binary success flag, $R=\mathrm{success}$. Advantages are the standard
group-standardized GRPO values
$A_i=(R_i-\mu)/(\sigma+\varepsilon)$, where $\mu$ and $\sigma$ are the mean and standard
deviation over the $G$ group members. Zero-variance groups carry no advantage and are removed by
the dynamic-sampling filter. This reproduces the SimpleVLA-RL configuration.
The peak success values reported by SimpleVLA-RL are not directly comparable with the maximum of
the mean training curve in our main paper because the two presentations aggregate checkpoints
differently. SimpleVLA-RL reports the best-success checkpoint from each of three seeds and then
aggregates the resulting scores, whereas our curves show the step-aligned mean over five seeds
throughout training. For example, SimpleVLA-RL reports a $64.1\%$ success rate on
\textit{Lift Pot}. In our five Binary GRPO runs, the best success rates are $60.6\%$ at step 39,
$64.5\%$ at step 34, $64.5\%$ at step 49, $63.7\%$ at step 29, and $64.8\%$ at step 44. These
per-seed maxima average to $63.6\%$, closely matching the reported result. However, because the
maxima occur at different training steps, they do not appear simultaneously in the step-aligned
five-seed mean curve and are therefore smoothed when the learning curves are averaged.

\paragraph{Binary RLOO.}
The reward is identical to Binary GRPO; the only change is the estimator. Advantages use the
leave-one-out baseline~\eqref{eq:app-rloo} with no standard-deviation division. Since Prism-GRPO
also uses RLOO, this baseline isolates the contribution of the quality term from that of the
estimator. Everything else is inherited unchanged.

\paragraph{Random reward.}
Each trajectory receives
$R=\mathrm{success}+\lambda u$, where
$u\sim\mathrm{Uniform}(0,1)$ is sampled independently for each rollout. The random term almost
always introduces within-group reward variation, so same-outcome groups rarely remain
zero-variance and are typically retained by the dynamic-sampling filter. However, because
$u$ contains no information about task execution, it provides only noise rather than a meaningful
quality signal. Random reward uses the same RLOO advantage estimator as Prism-GRPO, thereby
isolating the effect of replacing the trajectory-derived quality signal with uninformative random
scores while holding the reward scale, filtering behavior, and advantage estimator fixed.

\paragraph{RL-ZVP.}
RL-ZVP~\citep{le2025no} was developed for LLM reasoning, where one decision is one output token. We
port it to our discrete action head by treating one action-chunk step as one decision: we mean-pool
the per-token entropies over the $14$ action-token slots of a step to obtain a single entropy
$H_{i,t}$ per decision, computed from the rollout policy and detached before use. A group is judged
degenerate on the binary success signal, all-success or all-failure, rather than on reward
variance; this is exact for a binary base reward. On a degenerate group the zeroed GRPO advantage
is overwritten by
\begin{equation}
A_{i,t}=
\begin{cases}
+\alpha\,H_{i,t}, & \text{all-success},\\
-\alpha\,\bigl(\max_k H_{i,k}-H_{i,t}\bigr), & \text{all-failure},
\end{cases}
\end{equation}
with $\alpha=0.1$ and $\max_k$ over the valid decisions of the trajectory. Mixed groups keep their
standard GRPO advantage; RL-ZVP composes only with group-standardized GRPO. So that the override
sees degenerate groups, the dynamic-sampling filter is disabled for this baseline.

\section{Compute usage}
\begin{figure}[t]
  \centering
\includegraphics[width=0.85\linewidth]{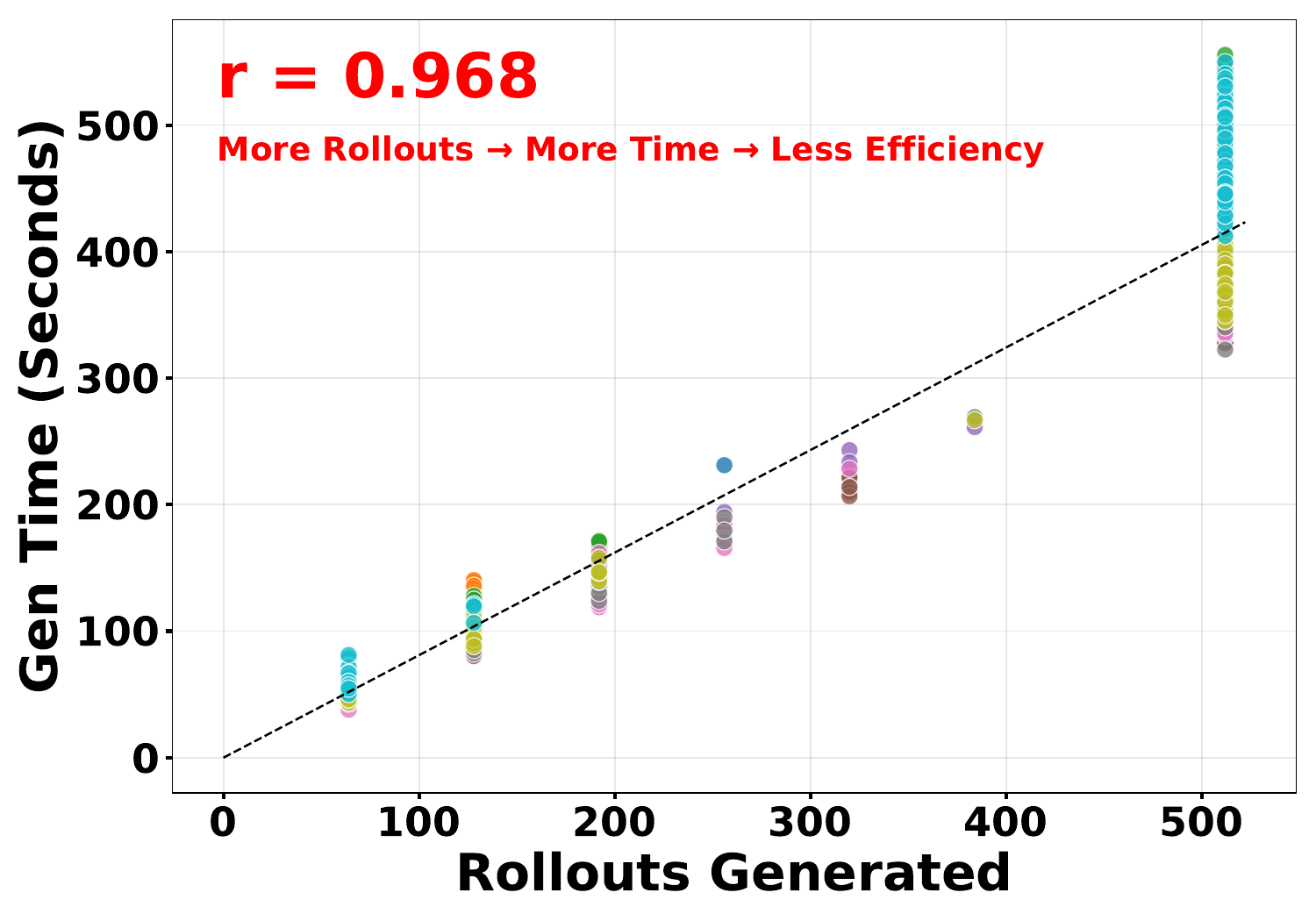}
  \caption{\textbf{Rollout-generation cost.}
  Per-step generation time scales nearly linearly with the number of requested rollouts across
  ten RL runs ($1{,}295$ measurements).  Each color denotes one run. } 
  \label{fig:gen-time-vs-rollouts}
\end{figure}

\paragraph{Hardware and per-run wall time.}
Each task is trained on a single node with $8\times$ NVIDIA H100-80\,GB GPUs. A complete RL run
takes approximately $12$--$16$ hours, depending on the task horizon and rollout cost. Rollout
generation accounts for roughly $8$--$12$ hours and dominates the total wall time, with the
remaining time spent on policy updates, dynamic filtering, and periodic validation.

\paragraph{Rollout savings translate directly to compute savings.}
Figure~\ref{fig:gen-time-vs-rollouts} shows a near-linear relationship between rollout count and
generation time, averaging approximately $0.81$\,s per rollout. Because rollout generation
dominates training time, reducing the rollout budget yields an almost proportional wall-clock
saving. For example, a $30\%$ reduction from the $35{,}000$ rollouts required by Binary GRPO on
\textit{Lift Pot} avoids $10{,}500$ rollouts, saving approximately $2.4$ of the $7.9$ hours spent
on rollout generation.

\begin{figure*}[t]
    \centering
    \includegraphics[width=\textwidth]{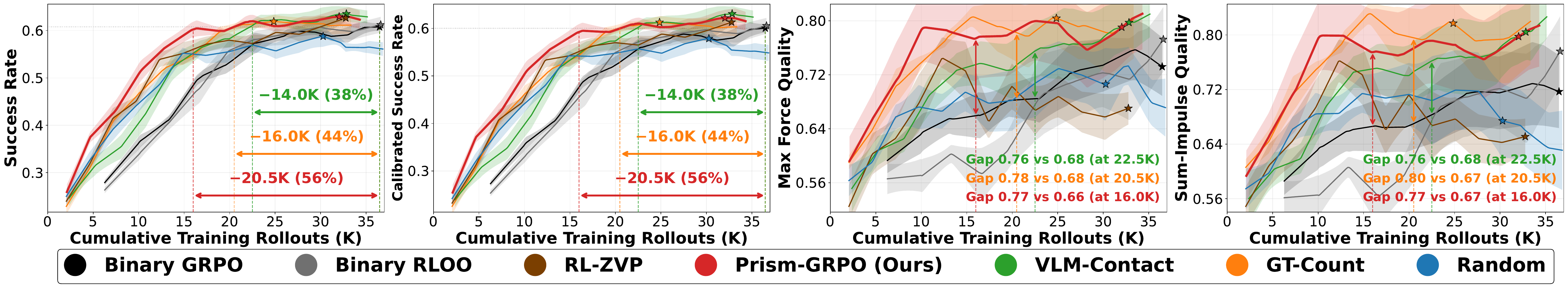}
    \caption{\textbf{Full collision-quality results.}}
    \label{fig:full_collision_metrics}
\end{figure*}

\begin{figure*}[t]
    \centering
    \includegraphics[width=\textwidth]{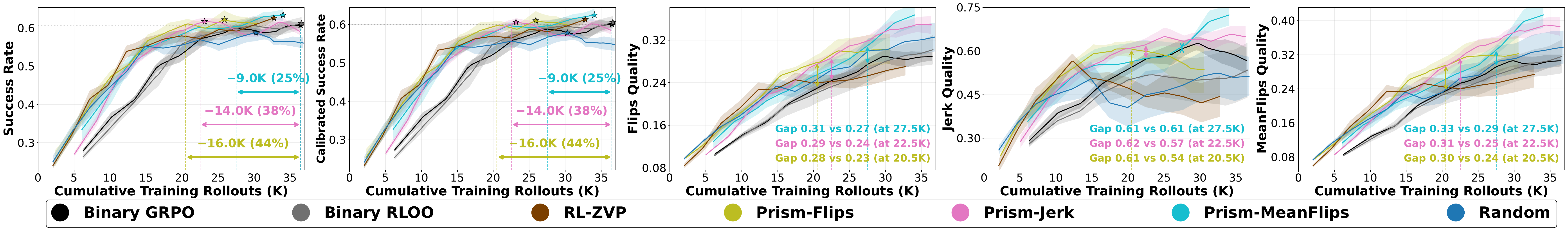}
    \caption{\textbf{Full smoothness-quality results.}}
    \label{fig:full_smoothness_metrics}
\end{figure*}

\begin{figure*}[t]
    \centering
    \includegraphics[width=\textwidth]{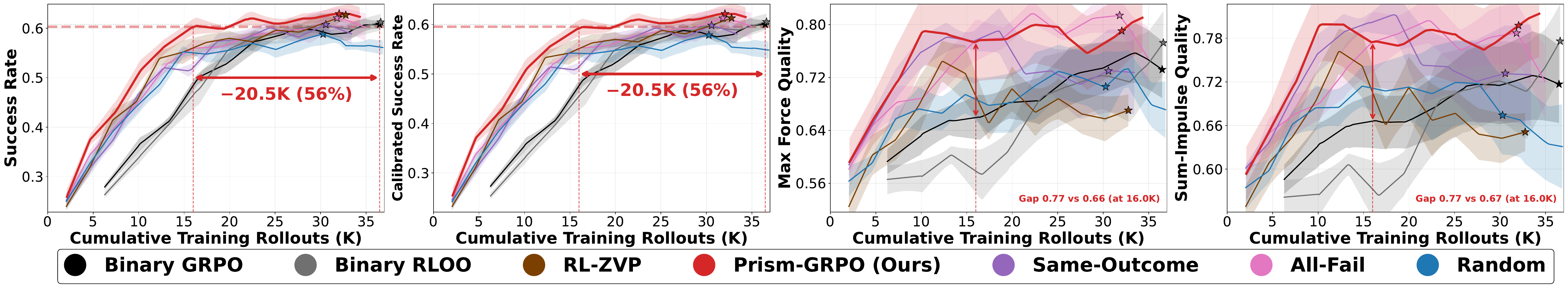}
    \caption{\textbf{Quality-term placement ablation.}
    Full results across success and collision-quality metrics.}
    \label{fig:term_ablation}
\end{figure*}
\begin{figure*}[t]
    \centering
    \includegraphics[width=\textwidth]{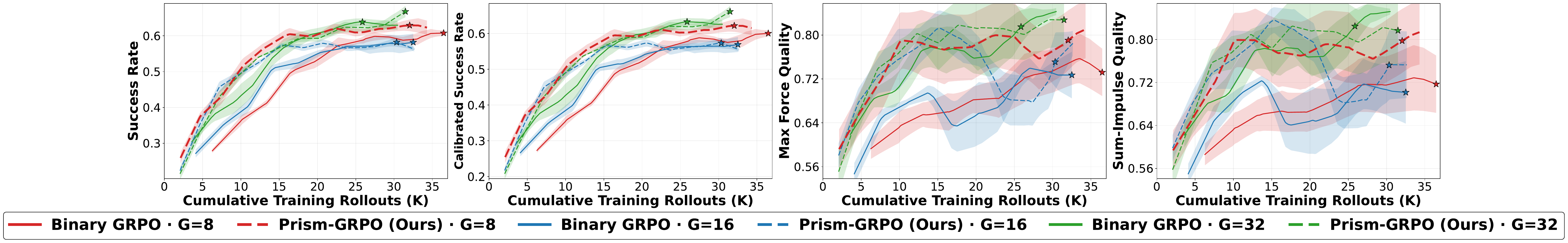}
    \caption{\textbf{Effect of group size.}
    Each optimizer update uses $512$ generated rollouts, while the group size
    varies over $G\in\{8,16,32\}$. Prism-GRPO provides the largest gain at
    $G=8$, while the gain decreases as the group size increases.}
    \label{fig:group_size_ablation}
\end{figure*}

\section{Detailed Experimental Results}

\label{app:exp}
\subsection{Full Statistics}
\label{app:g1}

Figure~\ref{fig:filter_statistics} examines how often the quality signal recovers same-outcome
groups and how this reduction in filtering affects rollout cost across all four tasks.

\paragraph{Filter discard rate.}
The filter discard rate is the fraction of generated groups whose final rewards are identical
across all group members and are therefore removed before the policy update. Under Binary GRPO and
Binary RLOO, every all-success or all-failure group has zero reward variance and is discarded.
Under Prism-GRPO, the quality term introduces reward variation whenever trajectories with the same
outcome differ in execution quality, allowing these groups to contribute to learning. The discard
rate is not necessarily zero because a same-outcome group can still receive identical quality
scores, such as an all-success group in which every trajectory has perfect quality. The reduction
in discard rate therefore directly measures how much otherwise unusable same-outcome data is
recovered by Prism-GRPO.

\paragraph{Rollouts per step.}
This metric counts every trajectory generated to construct one optimizer batch, including filtered
trajectories and surplus trajectories produced during adaptive gap filling. Each update requires
$64$ retained groups with $G=8$, corresponding to $512$ retained trajectories. The trainer first
generates $512$ trajectories and then samples additional scenes to replace discarded groups until
the target batch is complete. A higher discard rate therefore requires more refill rollouts and
increases the generation cost of each optimizer step.

\paragraph{Random and RL-ZVP incur no filtering.}
We omit Random and RL-ZVP from Figure~\ref{fig:filter_statistics} because neither method produces
discarded groups in our experiments. The filter removes a group only when all trajectories receive
the same final score. Random uses $R=\mathrm{success}+\lambda u$, where an independent continuous
random quality $u$ is sampled for each rollout, while RL-ZVP assigns trajectory-dependent
policy-confidence values. These scores vary within every sampled group, giving both methods a zero
discard rate and the minimum rollout cost of $512$ trajectories per step.

\paragraph{Full quality metrics.}
Due to space constraints, Figure~\ref{fig:ablations} in the main paper contains one collision panel
and one smoothness panel, each reporting a representative quality metric. Here, we provide the full
collision and smoothness results in Figures~\ref{fig:full_collision_metrics}
and~\ref{fig:full_smoothness_metrics}. For collision quality, Max-Force Quality and Sum-Impulse
Quality show similar trends despite measuring the strongest individual contact and the accumulated
contact burden, respectively. The smoothness metrics likewise produce consistent trends across
flip- and jerk-based measures. This cross-metric consistency indicates that the observed gains are
not specific to a single evaluation metric.

\subsection{Quality Weight Ablation}
\label{app:g2}
Figures~\ref{fig:q_ablation_lift} and~\ref{fig:q_ablation_move} report the full learning curves
for $\lambda\in\{0.2,0.5,0.9,1,2\}$ on \textit{Lift Pot} and \textit{Move Can Pot}. Each column
corresponds to one value of $\lambda$, and the rows report success rate, calibrated success rate,
Max-Force Quality, and Sum-Impulse Quality. On both tasks, success rate and calibrated success
rate follow similar trends, showing that the conclusions are not sensitive to contact-based
success calibration. Likewise, Max-Force Quality and Sum-Impulse Quality exhibit consistent
patterns despite measuring the strongest collision and cumulative collision burden, respectively.
The main paper uses Max-Force Quality as the primary collision-quality metric because it directly
corresponds to the default GT-Max Force training signal; the complete results here confirm that
the same conclusions hold under the other reported metrics.

\subsection{Quality-Term Placement Ablation}
\label{app:term}

\paragraph{The quality term should be applied to all groups.}
Figure~\ref{fig:term_ablation} compares our default reward with variants that apply the quality
term only to same-outcome groups or only to all-failure groups. Applying quality to every group
nearly doubles the success rate relative to these restricted variants and produces a more stable
quality curve. The quality term serves a useful role in every group type: in mixed groups, it
distinguishes trajectories within each outcome class while preserving success dominance; in
same-outcome groups, it provides the learning signal that the binary reward lacks. Restricting the
quality term therefore discards informative execution differences from part of the training data,
whereas the default combined reward uses them consistently across all groups.

\subsection{Effect of Group Size}
\label{app:group-size}

We ablate the number of rollouts sampled per group using
$G\in\{8,16,32\}$, where $G=8$ is the default setting in the main experiments.
To isolate the effect of group size, we keep the total number of generated
rollouts per optimizer update fixed at $512$ and adjust the number of groups
accordingly. All other settings, including initialization, rollout budget,
optimizer configuration, and quality weight, remain unchanged.

Figure~\ref{fig:group_size_ablation} shows that Prism-GRPO provides the largest
improvement at $G=8$. The gain becomes smaller at $G=16$, and the results are
nearly identical at $G=32$. This trend is consistent with the design of
Prism-GRPO: as $G$ increases, all-success and all-failure groups become less
frequent, so the binary baseline discards fewer groups and leaves less wasted
rollout data for Prism-GRPO to recover.

\begin{figure*}[t]
    \centering
    \includegraphics[width=\textwidth]{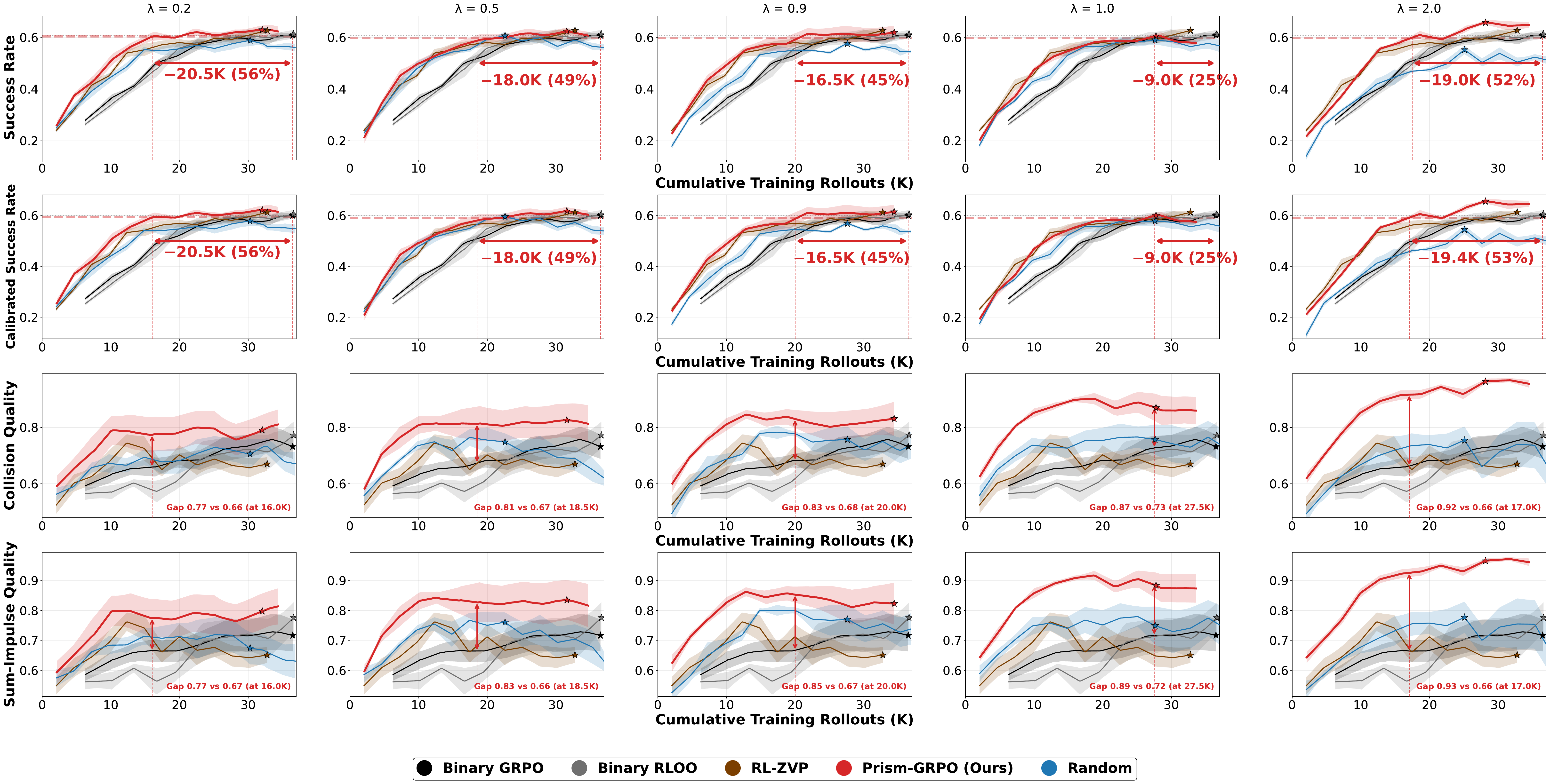}
    \caption{\textbf{Quality-weight ablation on \textit{Lift Pot}.}}
    \label{fig:q_ablation_lift}
\end{figure*}

\begin{figure*}[t]
    \centering
    \includegraphics[width=\textwidth]{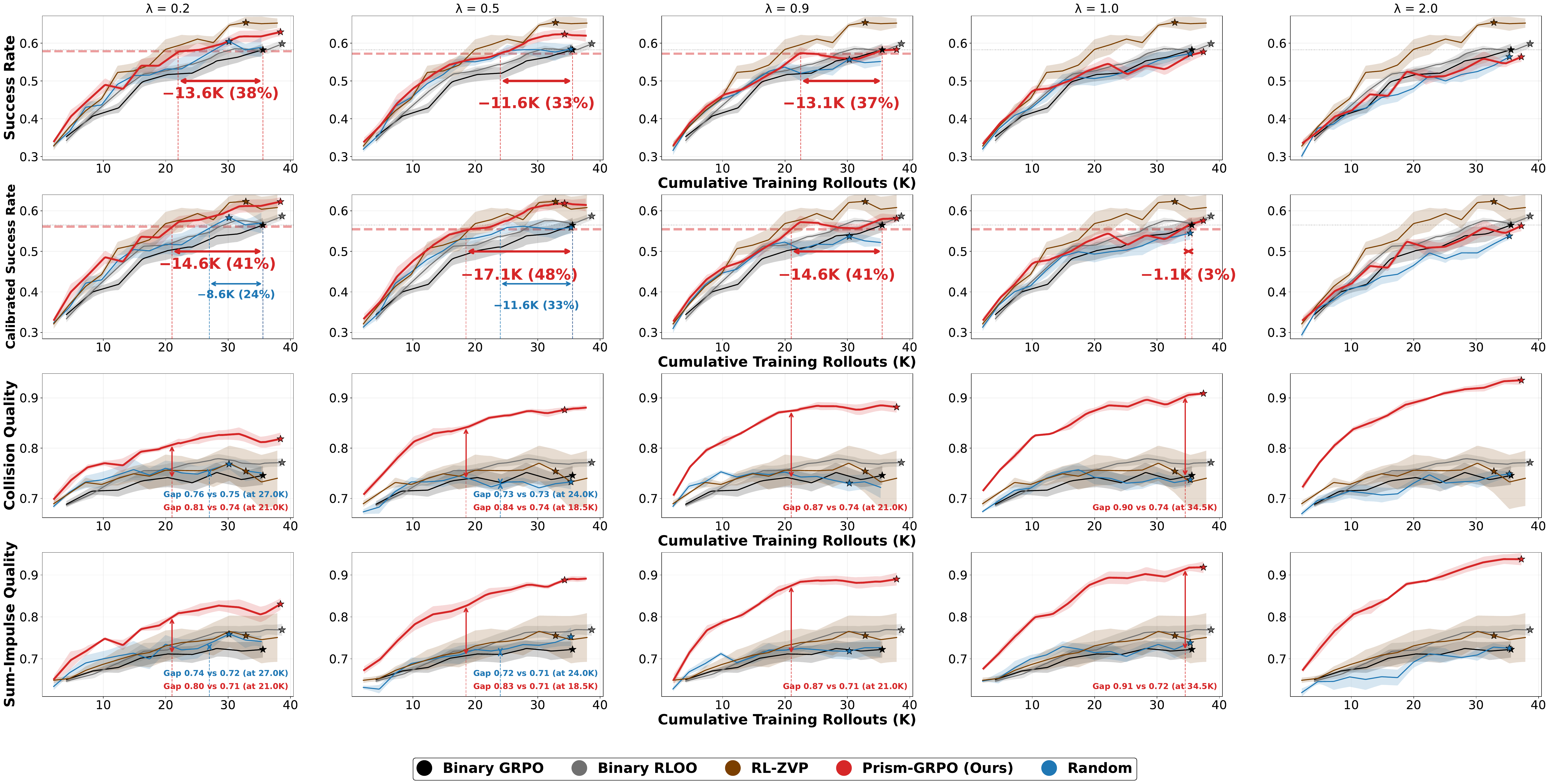}
    \caption{\textbf{Quality-weight ablation on \textit{Move Can Pot}.}}
    \label{fig:q_ablation_move}
\end{figure*}

\subsection{Effect of Quality-Weight Decay}
\label{app:decay}

A natural question is whether the quality weight $\lambda$ should decay during training. Early in
training, failures dominate and many groups are all-failure, so the quality term is particularly
useful for recovering learning signal. Later, as the success rate increases, one might expect that
reducing $\lambda$ would shift the update increasingly toward task success.
Figure~\ref{fig:decay_ablation} shows that decay does not improve performance for
$\lambda=0.2$ or $0.5$. For $\lambda=0.9$, decay reaches the target about $2$K rollouts earlier,
but the advantage is not clearly observable from the overall learning curves. Because decay
provides no substantial or consistent improvement, we use a fixed quality weight throughout the
main experiments.

\begin{figure*}[t]
    \centering
    \includegraphics[width=\textwidth]{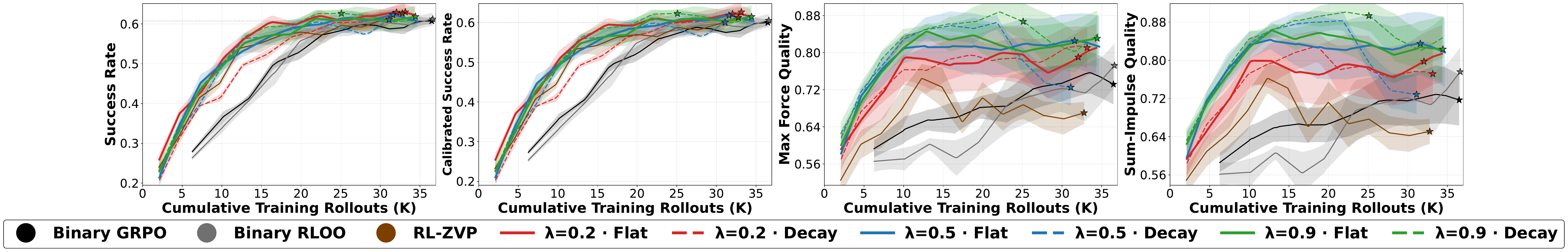}
\caption{\textbf{Quality-weight decay ablation.}
Solid lines use a fixed quality weight, dashed lines use decay, and colors indicate different
values of $\lambda$.}
    \label{fig:decay_ablation}
\end{figure*}

\subsection{Effect of the Advantage Estimator}
\label{app:rloo-lambda}

As discussed in Section~\ref{app:advantage}, RLOO preserves the scale of the quality term in
same-outcome groups, whereas standard group normalization largely cancels this scale. We therefore
compare the two estimators while keeping the reward and all other training settings unchanged.
Figure~\ref{fig:rloo_lambda} shows that RLOO produces more stable learning and consistently achieves
the expected rollout-saving gains. Standard group-normalized GRPO exhibits less stable optimization
and smaller, less consistent rollout savings. We therefore use RLOO throughout the main experiments.

\begin{figure*}[t]
    \centering
    \includegraphics[width=\textwidth]
    {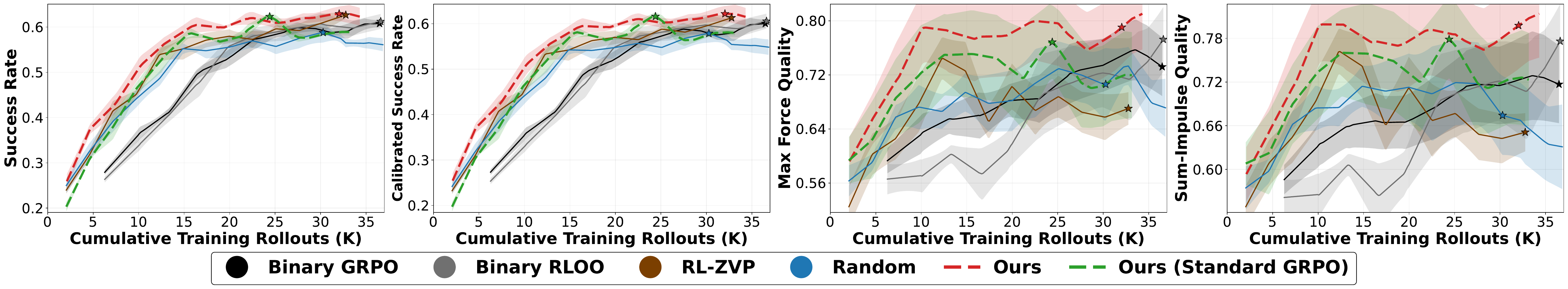}
    \caption{\textbf{Effect of the advantage estimator.}
    RLOO yields more stable learning and more consistent rollout savings than standard
    group-normalized GRPO. $\lambda=0.2$}
    \label{fig:rloo_lambda}
\end{figure*}

\section{Empirical Analysis of Success--Quality Alignment}
\label{app:correlation}
\subsection{Analysis Protocol}
\label{app:correlation_protocol}

We empirically examine whether the trajectory-level quality signals used by Prism-GRPO are
associated with task success. Correlations computed across all trajectories can be confounded by
scene difficulty, since easier scenes may simultaneously produce higher success rates and cleaner
execution. We therefore use two complementary diagnostics: a within-scene quality gap that compares
success and failure under the same scene configuration, and a cross-scene Spearman correlation
between scene success rate and mean quality.

\paragraph{Within-scene measure.}
For each task, we evaluate its best-success checkpoint on $256$ randomized scenes with $56$
stochastic rollouts per scene, yielding $14{,}336$ trajectories per task. Rollouts from the same
scene share an identical simulator initialization, including object poses, clutter, lighting,
tabletop properties, and instruction; only policy sampling and simulator stochasticity vary. We
apply no dynamic filtering and include every generated trajectory. For each scene containing at
least one successful and one failed rollout, we define the within-scene quality gap as
\begin{equation}
    \Delta q(x)
    =
    \bar q_{\mathrm{succ}}(x)
    -
    \bar q_{\mathrm{fail}}(x),
    \label{eq:within_scene_quality_gap}
\end{equation}
where the two terms denote the mean normalized quality of successful and failed trajectories,
respectively. A positive gap indicates that successful trajectories have higher quality under the
same scene configuration. Each mixed-outcome scene receives equal weight, and we report the mean
gap with a $95\%$ confidence interval obtained from $10{,}000$ scene-level bootstrap resamples.

\paragraph{Cross-scene measure.}
As a complementary population-level diagnostic, we compute Spearman's $\rho$ between each scene's
success rate and mean trajectory quality. Each point in the corresponding scatter plot represents
one scene, with the horizontal coordinate given by its success rate over the $56$ rollouts and the
vertical coordinate given by their mean quality. Spearman correlation measures whether scenes with
higher success rates also tend to have higher quality, without assuming a linear relationship.
Whereas the within-scene gap compares success and failure under an identical scene configuration,
the cross-scene measure evaluates whether quality tracks success across the broader task
distribution. We compute both diagnostics for the collision and smoothness signals defined in
Appendix~\ref{app:quality}.

\subsection{Collision-Based Alignment}
\label{app:correlation_collision}

Figure~\ref{fig:corr_collision} reports both alignment diagnostics for Max-Force Quality,
Contact-Count Quality, and Sum-Impulse Quality. Max-Force Quality is the default signal used in
the main experiments, while the other two signals test whether the conclusion depends on a
particular definition of unintended contact. For Max-Force Quality, the mean within-scene quality
gap is approximately $0.55$ on \textit{Lift Pot}, $0.45$ on \textit{Move Can Pot}, $0.87$ on
\textit{Handover Block}, and $0.85$ on \textit{Beat Block Hammer}, with $89.5\%$--$100\%$ of
mixed-outcome scenes exhibiting a positive gap. The corresponding cross-scene Spearman
correlations are also strongly positive, reaching $0.79$, $0.68$, $0.98$, and $0.80$,
respectively. Contact-Count and Sum-Impulse Quality show the same qualitative pattern across all
four tasks, despite measuring collision frequency and accumulated contact burden rather than the
strongest individual contact. Together, these results show that success is strongly associated
with cleaner execution both within identical scenes and across the task distribution, and that
this relationship is robust to the particular collision statistic extracted from the trajectory.

\begin{figure*}[t]
    \centering
    \includegraphics[width=\textwidth]{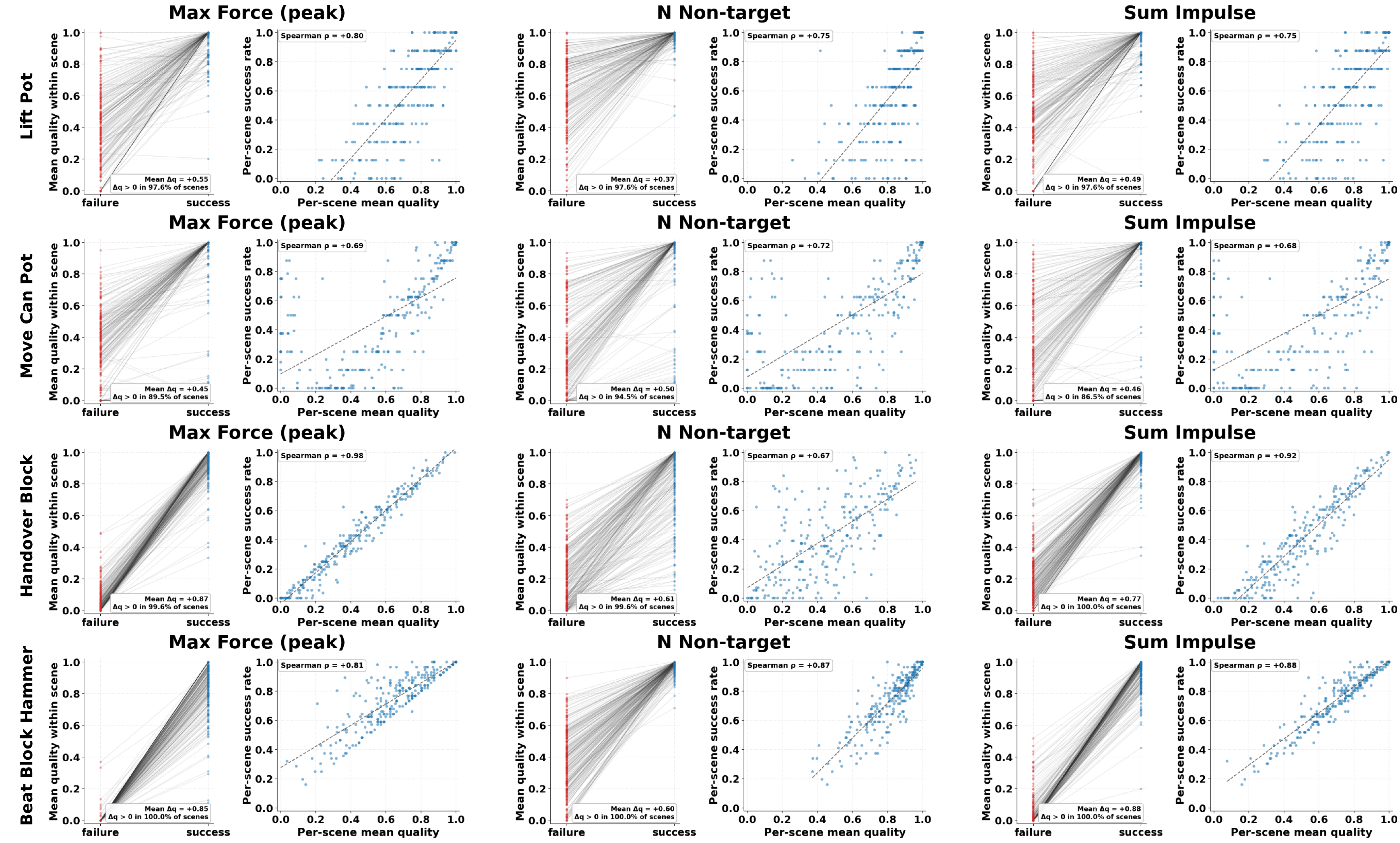}
    \caption{\textbf{Collision-based success--quality alignment.}
    Each task is evaluated using Max-Force, Contact-Count, and Sum-Impulse Quality. Within-scene
    plots compare the mean quality of successful and failed trajectories generated from the same
    mixed-outcome scene; positive gaps indicate that successful trajectories are cleaner.
    Cross-scene plots report the Spearman correlation between scene success rate and mean quality.}
    \label{fig:corr_collision}
\end{figure*}

\subsection{Smoothness-Based Alignment}
\label{app:correlation_smoothness}

Figure~\ref{fig:corr_smoothness} reports both alignment diagnostics for
\textbf{Prism-Flips}, \textbf{Prism-MeanFlips}, and \textbf{Prism-Jerk}, labeled as
Flips (Max), Mean Flips, and Max Jerk in the figure. Across all task--signal
combinations, $87\%$--$100\%$ of mixed-outcome scenes have a positive
within-scene quality gap, and the cross-scene correlations are consistently
positive and strong. These results show that the success--smoothness relationship
holds across multiple smoothness definitions rather than depending on a single
metric.

\begin{figure*}[t]
    \centering
    \includegraphics[width=\textwidth]{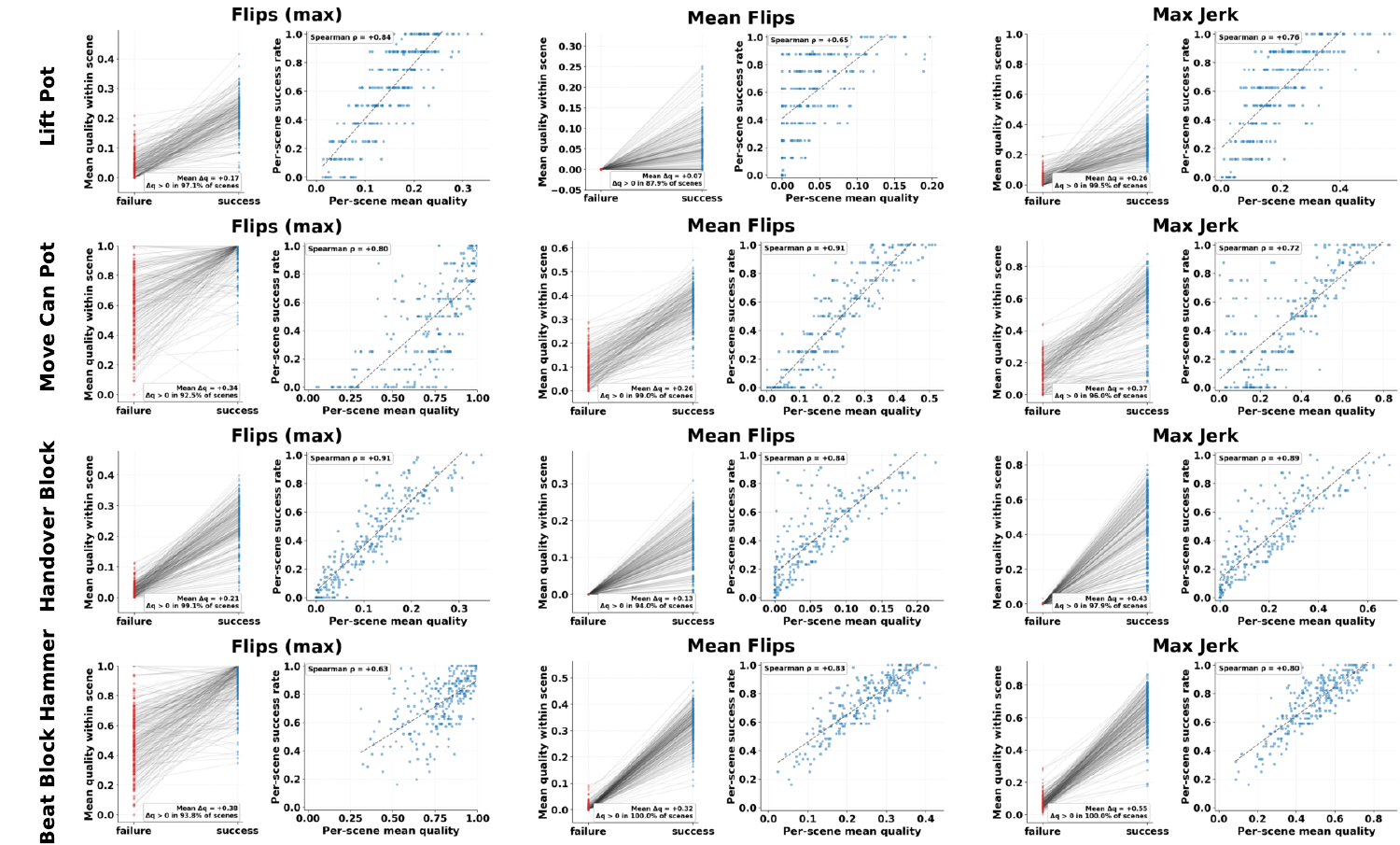}
    \caption{\textbf{Smoothness-based success--quality alignment.}
    Each task is evaluated using maximum joint reversals, mean joint reversals, and action jerk.
    The within-scene and cross-scene diagnostics follow the same definitions as
    Figure~\ref{fig:corr_collision}.}
    \label{fig:corr_smoothness}
\end{figure*}

One implementation detail is needed for \textbf{Prism-Flips} on
\textit{Move Can Pot} and \textit{Beat Block Hammer}, where successful execution naturally
contains several joint-direction reversals during grasp-and-place or lift-and-strike motions.
We therefore subtract a reversal floor automatically estimated as the median reversal count among
successful trajectories before normalization, allowing Prism-Flips to measure excess oscillation
beyond the motion required for task completion. The same data-driven estimation rule is applied
without manual tuning, while no floor is needed for \textit{Lift Pot} or
\textit{Handover Block}. Prism-MeanFlips and Prism-Jerk are computed without this adjustment.
Our goal is not to identify the optimal smoothness metric, but to test whether several reasonable
trajectory-derived signals are consistently associated with success; developing more specialized
measures that better separate necessary motion from undesirable oscillation remains an interesting
direction for future work.

\section{Real-World Experiments}
\label{app:real}

\subsection{Deployment Setup}

We evaluate the policies zero-shot on a Piper robot without additional real-world fine-tuning.
The task is to move a can next to a pot. Observations are captured by an Orbbec DaBai DC1
camera at $640\times480$ resolution and resized to the $240\times320$ resolution used during
training before being passed through the standard policy preprocessing pipeline.

The policy predicts $25$-step action chunks. Each executed action is seven-dimensional, consisting
of six arm-joint targets and one gripper target. To prevent abrupt motions on the physical robot,
we apply a proportional rate limit to the six arm-joint dimensions in each control tick:
\begin{equation}
\begin{aligned}
    \alpha
    &=
    \min\!\left(
        1,\,
        \frac{\eta}
        {\max_{1\leq j\leq 6}
        |\Delta_{j}^{\mathrm{model}}|}
    \right), \\
    \Delta_{i}^{\mathrm{sent}}
    &=
    \alpha\,\Delta_{i}^{\mathrm{model}},
    \qquad
    i\in\{1,\ldots,6\},
    \quad
    \eta=0.05\ \mathrm{rad}.
\end{aligned}
\end{equation}
Here, $\Delta^{\mathrm{model}}$ is the model-requested joint displacement and $\eta$ limits the
largest arm-joint change to $0.05$\,rad per control tick. When any requested joint displacement
exceeds this limit, all six arm-joint dimensions are scaled by the same factor, preserving the
relative direction of the commanded motion. The seventh action dimension controls the gripper and
is exempt from this rate limit. 

\subsection{Evaluation Protocol}

The purpose of this experiment is to examine whether the shove-cheat behavior observed in the
simulated \textit{Move Can Pot} task transfers to the real world. In this shortcut, the policy moves
the pot toward the can instead of lifting and placing the can as instructed. To stress-test this
behavior, for each method we select the checkpoint with the highest simulated shove-cheat rate
among checkpoints with similar simulated success rates of approximately $65\%$.
We conduct $25$ trials for each policy. In every trial, the can is initially placed to the left of
the pot. We manually reset the can and pot to approximately the same positions before each trial. 

\subsection{Real-World Results}

Table~\ref{tab:real_results} reports clean success and shove-cheat behavior over $25$ trials per
method. Prism-GRPO achieves $6/25$ clean successes, compared with $4/25$ for Binary GRPO and
$2/25$ for RL-ZVP. We do not interpret the difference between Prism-GRPO and Binary GRPO as
conclusive evidence of improved clean success, since their counts differ by only two trials. The
more pronounced result is the occurrence of the shove-cheat behavior observed in simulation.
Instead of moving the can beside the pot, the policy attempts to move the pot toward the stationary
can. This shortcut occurs in $5/25$ RL-ZVP trials and $1/25$ Binary GRPO trials, but is not
observed in any of the $25$ Prism-GRPO trials. Thus, while the clean-success comparison should be
interpreted cautiously, the shortcut pattern qualitatively supports the simulation finding that
Prism-GRPO discourages undesirable task-completion strategies.

\end{document}